\documentclass[11pt]{article}
\usepackage{amsmath, amssymb, amsthm}
\usepackage{graphicx}
\usepackage{authblk}
\usepackage{times}
\usepackage{hyperref}
\usepackage{natbib}
\usepackage{xcolor}
\usepackage{algorithm}
\usepackage{algpseudocode}
\usepackage{subcaption}

\theoremstyle{plain}
\newtheorem{theorem}{Theorem}
\newtheorem{proposition}{Proposition}
\newtheorem{lemma}{Lemma}

\theoremstyle{definition}

\newtheorem{condition}{Condition}
\newcommand{\op}{\mathrm{op}}

\theoremstyle{remark}
\newtheorem{remark}{Remark}
\title{\textbf{When Does Learning Renormalize? Sufficient Conditions for Power–Law Spectral Dynamics}}

\author[1]{Yizhou Zhang\footnote{Corresponding Author}}
\affil[1]{zyizhou96@gmail.com}
\date{}

\begin{document}
\maketitle

\begin{abstract}
Empirical power–law scaling has been widely observed across modern deep learning systems, yet its theoretical origins and scope of validity remain incompletely understood. The Generalized Resolution–Shell Dynamics (GRSD) framework models learning as spectral energy transport across logarithmic resolution shells, providing a coarse–grained dynamical description of training. Within GRSD, power–law scaling corresponds to a particularly simple renormalized shell dynamics; however, such behavior is not automatic and requires additional structural properties of the learning process.

In this work, we identify a set of sufficient conditions under which the GRSD shell dynamics admits a renormalizable coarse–grained description. These conditions constrain the learning configuration at multiple levels, including boundedness of gradient propagation in the computation graph, weak functional incoherence at initialization, controlled Jacobian evolution along training, and log–shift invariance of renormalized shell couplings. We further show that power–law scaling does not follow from renormalizability alone, but instead arises as a rigidity consequence: once log–shift invariance is combined with the intrinsic time–rescaling covariance of gradient flow, the renormalized GRSD velocity field is forced into a power–law form.

Beyond the theoretical analysis, we provide direct empirical evidence that the key structural requirement—log–shift invariance of renormalized shell couplings—is approximately realized in modern residual architectures, and is significantly degraded in non–residual counterparts. These experiments validate the non–vacuous nature of the sufficient conditions and clarify the architectural mechanisms that promote renormalizable spectral transport.

We emphasize that the conditions identified here are sufficient but not necessary, and we do not claim universality of power–law scaling across all deep learning settings. Rather, our results delineate a class of learning dynamics in which renormalizable shell transport and power–law behavior can be derived from structural and dynamical principles. By framing neural scaling laws within a renormalization–theoretic perspective, this work provides a principled foundation for understanding when and why power–law scaling should be expected to hold, as well as when it may break down.
\end{abstract}

\section{Introduction}

Empirical power--law scaling has emerged as one of the most robust and
reproducible phenomena in modern deep learning.
Across model families, tasks, and training regimes, performance metrics
such as loss, error, or perplexity often exhibit smooth power--law
dependence on model size, dataset size, or compute budget
\citep{hestness2017deep,kaplan2020scaling,henighan2020scaling,hoffmann2022training}.
More recently, scaling behavior has been observed to persist beyond
classical regimes, extending to precision, data pruning, and
architectural interventions
\citep{sorscher2022beyond,kumar2024scaling}.
Despite its empirical ubiquity, the theoretical origins of these power
laws remain only partially understood.

A growing body of work has sought to explain neural scaling laws through
kernel limits and mean--field approximations
\citep{jacot2018neural,lee2019wide,yang2021tensor},
as well as through dynamical models that interpolate between lazy and
feature--learning regimes
\citep{bordelon2024dynamical,bordelon2023feature,domine2024lazy}.
While these approaches provide valuable insight into specific limits or
architectures, they do not by themselves explain why power--law behavior
arises so broadly, nor under what conditions it should be expected to
fail.
In particular, empirical evidence clearly indicates that scaling laws do
not hold universally: depending on architecture, optimization stability,
data distribution, or training regime, scaling behavior may degrade,
break, or transition between distinct regimes
\citep{sorscher2022beyond,keskar2017large,ghorbani2019investigation}.

In recent work, the Generalized Resolution--Shell Dynamics (GRSD)
framework was proposed as a spectral and operator--theoretic description
of learning dynamics in deep networks \citep{zhang2025generalized}.
GRSD formulates learning as an energy transport process across
logarithmic spectral shells, leading to a coarse--grained,
one--dimensional conservation law for error energy.
Within this framework, empirical neural scaling laws correspond to a
particularly simple form of the renormalized shell velocity: a power--law
function of the spectral coordinate.
However, in GRSD this power--law form is not automatic.
Rather, it reflects a nontrivial structural property of the learning
process—namely, that the shell dynamics admits a renormalizable closure
in the sense familiar from statistical physics and turbulence theory
\citep{kadanoff1966scaling,wilson1983renormalization,kolmogorov1995turbulence,forster1977large,spohn2012large}.

The central question addressed in this paper is therefore the following:
\emph{under what conditions does the GRSD shell dynamics admit a
power--law renormalized description?}
Our goal is not to claim that all deep learning systems satisfy such
conditions, nor that power--law behavior is universal.
Indeed, the empirical literature suggests the opposite: scaling laws are
contingent and can fail outside specific regimes.
Instead, we seek to identify a coherent and mathematically well--defined
set of \emph{sufficient conditions} under which renormalizable shell
dynamics can be established, and under which power--law behavior follows
as a rigidity consequence.

A central difficulty in this program is that some of the required
conditions—most notably log--shift invariance of renormalized shell
couplings—are genuinely structural and do not follow from generic
stability or locality considerations alone.
Such conditions would be of limited interest if they were never even
approximately realized in practical learning systems.
For this reason, an essential part of our analysis is to identify
concrete architectural mechanisms that promote these properties, and to
empirically test whether they are approximately satisfied in realistic
training settings.

Concretely, we study learning configurations defined by a network
architecture, an initialization scheme, and an optimization trajectory.
We provide a theorem establishing that the GRSD shell dynamics admits a
power--law velocity field when a collection of structural and dynamical
conditions are satisfied.
These conditions include
(i) locality in the gradient computation graph,
(ii) weak functional incoherence at initialization,
(iii) controlled evolution of the Jacobian along training, and
(iv) log--shift invariance of the renormalized shell couplings.
Crucially, power--law scaling does not follow from renormalizability or
log--shift invariance alone; rather, it emerges from a rigidity
mechanism once these properties are combined with the intrinsic
time--rescaling covariance of gradient flow.

At first sight, the sufficient conditions identified in this work may
appear strong or even restrictive.
Indeed, they impose nontrivial requirements on architectural structure,
initialization geometry, and training stability.
However, it is precisely these requirements that align strikingly well
with several core design principles of modern deep learning systems.

Contemporary architectures are overwhelmingly engineered to avoid
uncontrolled recurrence or long--range instantaneous coupling, favoring
feedforward or residual structures with bounded gradient propagation.
Training pipelines are carefully designed to ensure stability and
controllability, with explicit emphasis on preventing gradient explosion
or catastrophic spectral reorganization.
Likewise, random or weakly correlated initialization schemes are
ubiquitous, reflecting an implicit preference for functional
incoherence at the start of training.
Viewed through the GRSD lens, these widely adopted engineering choices
correspond closely to Conditions~\ref{cond:banded-jacobian}--\ref{cond:controlled-path}, which enforce locality, stability,
and statistical regularity of the induced operator dynamics.

From this perspective, the emergence of neural scaling laws should not
be interpreted as an intrinsic or inevitable property of deep learning
models.
Rather, it reflects the fact that engineering requirements of
trainability, stability, and scalability progressively constrain
learning systems toward a dynamical regime with no intrinsic spectral
scale and a renormalizable coarse--grained description.
Once learning dynamics enters this approximately scale--free and
renormalizable region, power--law behavior is no longer a matter of
modeling choice or empirical coincidence, but instead arises as a
rigidity consequence of the underlying dynamics.

In addition to the theoretical analysis, we empirically probe the
spectral transport operators induced during training and directly
measure the resulting coarse--grained shell couplings.
These experiments demonstrate that residual architectures approximately
realize the log--shift invariance required by the theory, while
structurally similar non--residual architectures do not.
This contrast provides direct evidence that the sufficient conditions
identified here capture meaningful architectural distinctions, rather
than abstract or vacuous assumptions.

Importantly, the conditions identified here are sufficient but not
necessary.
We do not exclude the possibility that other mechanisms—such as
stochastic optimization effects or alternative sources of effective
regularization—may also lead to scaling behavior.
Rather, our objective is to construct a principled framework in which
renormalizability and power--law scaling can be derived from structural
and dynamical assumptions, rather than postulated a priori.

A key feature of our analysis is that the sufficient conditions are
expressed at the level of operator dynamics and gradient evolution,
rather than being tied to a specific model family.
As a result, they can be meaningfully interpreted across a range of
modern architectures—including multilayer perceptrons, convolutional
networks, transformers, and structured state--space models—provided
these architectures satisfy the stated assumptions.
We emphasize that this correspondence is not universal: the presence of
architectural features aligned with our conditions does not imply that
scaling must occur, but rather that scaling is structurally permitted
within the GRSD framework.

From a broader perspective, our results position neural scaling laws
within the classical theory of renormalization and large--scale
dynamics.
Just as power laws in turbulence or critical phenomena arise when
microscopic details become irrelevant under coarse--graining, GRSD
predicts power--law learning dynamics when the learning process itself
admits a renormalizable and log--shift invariant closure compatible with
gradient--flow covariance.
Viewed through this lens, the observed success—and occasional
failure—of scaling laws in deep learning reflects whether the underlying
learning configuration lies within a renormalizable universality class.

The remainder of the paper is organized as follows.
Section~\ref{sec:grsd-background} reviews the GRSD framework and the
formulation of shell dynamics.
Section~\ref{sec:sufficient-conditions} states the sufficient conditions
for renormalizable shell dynamics and presents the main theorem.
Section~\ref{sec:residual-condition} provides a concrete structural
mechanism—based on residual learning—by which the most nontrivial
condition can be realized.
Section~\ref{sec:interpretation} interprets these conditions and
discusses their relationship to common deep learning design choices.
Section~\ref{sec:experiments} presents empirical validation of the
structural assumptions underlying the theory.
Finally, Section~\ref{sec:discussion} discusses limitations,
non--universal regimes, and directions for future work.

\section{Generalized Resolution--Shell Dynamics}
\label{sec:grsd-background}

In this section we briefly review the Generalized Resolution--Shell Dynamics (GRSD) framework, which provides the dynamical and spectral setting for the results of this paper.
Our presentation is intentionally concise and focuses only on the aspects of GRSD that are required to formulate and analyze power--law renormalizability.
We refer to \citep{zhang2025operator} for a detailed and comprehensive treatment.

\subsection{Spectral shells and shell energies}
\label{sec:grsd-shells}

Consider a supervised learning problem with model parameters $\theta(t)$ trained by gradient-based optimization.
Let $J(t)$ denote the Jacobian of the model outputs with respect to parameters, viewed as a linear operator from parameter space to function space.
We define the associated positive semidefinite operator
\begin{equation}
M(t) := J(t) J(t)^{\ast},
\end{equation}
which plays the role of a time-dependent kernel governing error dynamics.
Operators of this form have been extensively studied in learning theory and operator-based analyses of neural networks
\citep{rosasco2010learning,koltchinskii2017normal}.

GRSD analyzes learning dynamics through the spectral decomposition of $M(t)$.
Let $\lambda$ denote the spectral coordinate of $M(t)$, and introduce logarithmic spectral shells
\[
S_\alpha := \{\lambda : s_\alpha \le \log \lambda < s_{\alpha+1}\},
\]
where $\{s_\alpha\}$ is a uniform partition of the log-spectrum.
For each shell $S_\alpha$, GRSD defines a shell energy $E_\alpha(t)$ corresponding to the error energy carried by modes within that spectral band.

Passing to a shell-averaged continuum description yields an energy density $\varepsilon(\lambda,t)$, defined as the piecewise-constant interpolation of $E_\alpha(t)$ over logarithmic shells.
This representation enables a coarse-grained description of learning dynamics in which microscopic spectral details are suppressed while large-scale spectral structure is retained.

\subsection{Renormalized shell dynamics}
\label{sec:grsd-renormalization}

A central result of GRSD is that, under mild structural assumptions, the shell energies satisfy an approximate one-dimensional balance law of conservative form.
Specifically, the shell-averaged energy density obeys
\begin{equation}
\label{eq:grsd-conservation}
\partial_t \varepsilon(\lambda,t) + \partial_\lambda J(\lambda,t)
= -D(\lambda,t),
\end{equation}
where $J(\lambda,t)$ is a spectral flux density describing energy transport across scales, and $D(\lambda,t)$ collects dissipative contributions.
This equation is the learning-theoretic analogue of energy cascade equations in turbulence and large-scale interacting systems
\citep{kolmogorov1995turbulence,forster1977large,spohn2012large}.

The GRSD formulation does not, by itself, specify the functional form of the flux $J(\lambda,t)$.
Instead, it expresses learning as a transport process whose qualitative behavior depends on how $J$ relates to the local energy density $\varepsilon$.
A particularly simple and analytically tractable regime arises when the dynamics admits a \emph{renormalized velocity field}
\begin{equation}
\label{eq:grsd-velocity}
v(\lambda,t) := \frac{J(\lambda,t)}{\varepsilon(\lambda,t)},
\end{equation}
which depends on $\lambda$ through a power law.
In this case, the shell dynamics becomes renormalizable in the classical sense: under spectral coarse-graining, the functional form of the evolution equation is preserved up to rescaling.

Importantly, such power-law renormalizability is \emph{not automatic}.
The GRSD framework allows for general, scale-dependent fluxes that may involve multiple characteristic scales or exhibit broken scaling behavior.
Consequently, the appearance of power-law velocity fields reflects additional structural properties of the learning configuration rather than a generic consequence of spectral coarse-graining.

The goal of this paper is to make this distinction explicit.
Rather than postulating power-law renormalizability as an assumption, we ask under what conditions on the learning configuration—encompassing architecture, initialization, optimization stability, and statistical scale relations—the GRSD shell dynamics necessarily admits a power-law renormalized description.
The next section formalizes this question and presents a set of sufficient conditions under which such renormalizability can be rigorously established.

\section{Sufficient Conditions for Power--Law Renormalizability}
\label{sec:sufficient-conditions}

This section formulates a set of sufficient conditions under which the GRSD shell dynamics admits a power--law renormalized description.
We begin by formalizing the notion of a learning configuration and its associated Jacobian dynamics.
We then state the sufficient conditions and present the main theorem.
Interpretation and architectural implications are deferred to Section~\ref{sec:interpretation}.

\subsection{Learning configurations and Jacobian dynamics}
\label{sec:learning-config}

We consider supervised learning dynamics parameterized by time $t \in [0,T]$, induced by a learning configuration consisting of:
(i) a network architecture specifying a parameter-to-function map,
(ii) an initialization scheme for the parameters,
and (iii) an optimization trajectory generated by gradient-based training.

Let $J(t)$ denote the Jacobian of the model outputs with respect to parameters, viewed as a linear operator from parameter space to function space.
As in GRSD, we define the associated positive semidefinite operator
\begin{equation}
M(t) := J(t) J(t)^{\ast},
\end{equation}
which governs the instantaneous learning dynamics in function space.
Operator evolutions of this form arise naturally in both kernel limits and feature-learning regimes
\citep{jacot2018neural,lee2019wide,bordelon2024dynamical,zhang2025operator}.

To capture structural locality in the gradient computation graph, we decompose the Jacobian into a collection of blocks,
\begin{equation}
J(t) = \bigl(J^{(1)}(t), J^{(2)}(t), \dots, J^{(L)}(t)\bigr),
\end{equation}
where each block corresponds to a contiguous subgraph of the computation graph (e.g., layers, modules, or residual units).
This decomposition is assumed to be fixed throughout training and serves only as a bookkeeping device for expressing locality and coherence properties.
No assumption is made that the blocks correspond to distinct spectral shells.

The evolution of $J(t)$ along the optimization trajectory induces an evolution of $M(t)$ and, through its spectrum, the GRSD shell dynamics reviewed in Section~\ref{sec:grsd-background}.
Our goal is to characterize conditions under which this induced shell dynamics is renormalizable with a power--law velocity field.

\subsection{Statement of sufficient conditions}
\label{sec:conditions}

We now state a set of sufficient conditions under which power--law renormalizability
of GRSD shell dynamics can be rigorously derived.
These conditions constrain the learning configuration at the level of computation
graph structure, initialization geometry, training stability, and scale consistency.
They are not claimed to be necessary, nor are they asserted to hold universally across
all deep learning settings.

\begin{condition}[Graph--banded Jacobian evolution]
\label{cond:banded-jacobian}
There exists a constant $K = O(1)$ such that for each block index $l$ and all $t \in [0,T]$,
\[
\dot J^{(l)}(t) \in \mathrm{span}\bigl\{ J^{(m)}(t) : |m-l| \le K \bigr\}.
\]
\end{condition}

Condition~\ref{cond:banded-jacobian} expresses locality of gradient propagation in the
computation graph.
It allows for recurrent or cyclic dependencies provided that their influence remains
confined to a bounded neighborhood and does not induce long-range instantaneous coupling
across blocks.

\begin{condition}[Initial functional incoherence]
\label{cond:init-incoherence}
There exists a nonnegative sequence $\{\varepsilon_k\}_{k \ge 1}$ with
$\sum_{k \ge 1} \varepsilon_k < \infty$ such that
\[
\bigl\| J^{(l)}(0)^{\ast} J^{(m)}(0) \bigr\|_{\mathrm{op}}
\;\le\;
\varepsilon_{|l-m|}
\qquad
\text{for all } l,m .
\]
\end{condition}

Condition~\ref{cond:init-incoherence} requires that long-range correlations between
Jacobian blocks at initialization are summable.
It does not preclude structured local dependencies or short-range correlations,
and is naturally satisfied by a wide class of random or weakly correlated initializations.

\begin{condition}[Controlled Jacobian path and statistical regularity]
\label{cond:controlled-path}
There exists a constant $C_J < \infty$ such that
\[
\sup_{t \in [0,T]}
\Bigl(
\|J(t)\|_{\mathrm{op}} + \|\dot J(t)\|_{\mathrm{op}}
\Bigr)
\;\le\;
C_J .
\]
Moreover, on any intermediate spectral window away from the extreme edges,
the Jacobian-induced operators and their spectral projectors evolve without
abrupt reorganization: all quadratic statistics entering the GRSD closure
remain uniformly bounded, admit Lipschitz dependence in logarithmic spectral
coordinates, and their shell-wise averages concentrate around their expectations
at rates compatible with the width and depth scaling limits.
\end{condition}

Condition~\ref{cond:controlled-path} enforces stability and regularity of the learning
trajectory.
Beyond uniform operator norm control, it rules out catastrophic spectral
rearrangements and ensures that the statistical quantities required for
renormalized shell closure are well-defined and self-averaging along the
optimization path. 

Assumptions of this type—uniformly bounded Jacobians/gradients along the optimization trajectory, sometimes supplemented by smoothness controls—are standard in analyses of overparameterized learning and trainability; see, e.g., bounded-Jacobian conditions in \cite{terjek2022framework} and Jacobian-spectrum stability via dynamical isometry in \cite{pennington2017resurrecting,tarnowski2019dynamical}.

\begin{condition}[Log--shift invariance of renormalized shell couplings]
\label{cond:log-shift-invariance}

Let $s=\log\lambda$, and let $\{B_i\}$ denote logarithmic spectral shells of
fixed width $h$ in $s$, centered at $s_i=ih$.
Let $\widehat{\Omega}_{ij}$ denote the shell-level renormalized coupling,
obtained by coarse-graining the mode-level operator $\dot M$ over shells
$B_i$ and $B_j$.

On an intermediate spectral window, the shell-level statistics satisfy Log--shift invariance, i.e. 
$\widehat{\Omega}_{ij}$ follows scale-free structure:
\[
\widehat{\Omega}_{ij}(t) =
\begin{cases}
\mathcal{K}_h\!\left((j-i)h\right) + \mathrm{err}_{ij},
& i\neq j,\\[1ex]
c\lambda_i + \mathrm{err}_{ii},
& i=j,
\end{cases}
\]
where $\mathcal{K}_h$ depends only on the relative log-scale separation,
and the residual terms $\mathrm{err}_{ij}$ vanish in the joint limit of large
model size and fine shell resolution. 

\end{condition}

Condition~\ref{cond:log-shift-invariance} requires the absence of any intrinsic absolute
scale in the renormalized spectral couplings beyond relative separations in
logarithmic scale. It forces that $\dot s = \frac{\dot\lambda}{\lambda}$ is invariant to the absolute scale of $\lambda$ (detailed proof provided in \ref{app:scale_invariant_sdot}).
Unlike Conditions~\ref{cond:banded-jacobian}--\ref{cond:controlled-path}, this condition
does not follow from generic stability or locality considerations and constitutes
a genuinely nontrivial structural requirement.

We emphasize that Condition~\ref{cond:log-shift-invariance} is not intended as a verifiable microscopic assumption with necessary and sufficient criteria. Rather, it should be interpreted as an effective principle characterizing learning configurations that operate near a renormalization fixed point in the GRSD sense. In practice, such configurations may be approached through architectural design (e.g., residual parameterizations), optimization choices, data preprocessing, and hyperparameter tuning, without being exactly realized. In this sense, Condition~\ref{cond:log-shift-invariance} plays a role analogous to idealized principles such as reversibility in Carnot engines or scale invariance at criticality: it defines a limiting structure that constrains admissible large-scale dynamics when approximately satisfied.

Together, Conditions~\ref{cond:banded-jacobian}--\ref{cond:log-shift-invariance}
ensure that GRSD shell dynamics admits a closed, scale-consistent, and renormalizable
description under logarithmic spectral coarse-graining.

\subsection{Main theorem}
\label{sec:main-theorem}

We are now ready to state the main result of this paper.

\begin{theorem}[Power--law renormalizability of GRSD shell dynamics]
\label{thm:power-law-renormalizability}
Fix a learning configuration and consider the induced GRSD shell dynamics defined on logarithmic spectral shells.
Suppose Conditions~\ref{cond:banded-jacobian}--\ref{cond:log-shift-invariance} hold on a training horizon $t \in [t_0,T]$.
Assume in addition the standard GRSD structural property that intra--shell couplings are antisymmetric, so that shell--internal transfers cancel in the energy balance.

Then the renormalized GRSD velocity field
\[
v(\lambda,t) := \frac{J(\lambda,t)}{\varepsilon(\lambda,t)}
\]
is uniquely constrained to the form
\[
v(\lambda,t)=\frac{c_0}{t}\lambda,
\qquad t\in[t_0,T],
\]
for a scalar coefficient $c_0=\frac{v(\lambda,t_0)}{\lambda}$.
\end{theorem}

\begin{remark}[Effective time and learning--rate schedules]
Throughout this work, the time variable $t$ denotes the \emph{effective time}
associated with the continuous--time gradient flow
\[
\dot\theta(t) = - \nabla_\theta \mathcal L(\theta(t)),
\]
which corresponds to a constant learning rate and no explicit scheduler.
All covariance and scaling properties used in our analysis are derived directly
from this ODE.

When a learning--rate schedule $\eta(t)$ is present, the parameter dynamics take
the form $\dot\theta = -\eta(t)\nabla_\theta \mathcal L$ and can be mapped to the
above gradient--flow form by introducing a reparameterized time \cite{lifunctional}
\[
\tau(t) = \int_0^t \eta(s)\,ds .
\]
All results in this paper then apply with $t$ replaced by the effective time
$\tau$.
\end{remark}

\subsection{Proof overview}
\label{sec:proof-overview}

We briefly outline the structure of the proof; technical details are deferred to the appendix.
First, Conditions~\ref{cond:banded-jacobian}--\ref{cond:controlled-path} imply that locality and summable incoherence of Jacobian blocks propagate along the training trajectory.
At the level of the error dynamics $\dot e = -M e$ with $M=JJ^*$, a Taylor expansion around a fixed log--shell basis shows that direct off-diagonal error exchange between shells separated by more than one shell index arises only at second order in variations of $M$ \cite{tian2025provablescalinglawsfeature}.
Consequently, the leading--order shell dynamics is governed solely by nearest--neighbor boundary fluxes, while nonlocal contributions enter only as subleading renormalizations (see Appendix~\ref{app:proof_power_law}).

Second, Condition~\ref{cond:log-shift-invariance} implies homogeneity of shell--averaged energy and flux densities under spectral rescaling.
Finally, homogeneity of the ratio $v(\lambda,t)=J(\lambda,t)/\varepsilon(\lambda,t)$ enforces a power--law functional form.
Detailed proofs are provided in Appendix~\ref{app:proof_power_law}.

\section{Structural sufficient condition for Condition~\ref{cond:log-shift-invariance}}
\label{sec:residual-condition}

Condition~\ref{cond:log-shift-invariance} plays a fundamentally different role from
Conditions~\ref{cond:banded-jacobian}--\ref{cond:controlled-path}.
While the latter encode generic locality, stability, and regularity properties that
are naturally satisfied by a wide class of modern block--stacked architectures,
Condition~\ref{cond:log-shift-invariance} imposes a genuinely nontrivial structural
requirement: the absence of any intrinsic absolute scale in the renormalized spectral
couplings beyond relative separations in logarithmic scale.

In particular, log--shift invariance does not follow from graph locality, controlled
optimization paths, or statistical self--averaging alone.
Establishing this condition therefore requires additional architectural structure.
In this section, we show that \emph{residual learning}, when combined with sufficient
depth, provides a concrete mechanism by which
Condition~\ref{cond:log-shift-invariance} can be rigorously realized.

\subsection{Parameter Jacobian structure in residual learning}
\label{subsec:residual-definition}

We begin by formalizing the Jacobian structure induced by residual learning
configurations at the level of parameters.

Let $\theta=(\theta_1,\dots,\theta_L)$ denote the decomposition of model parameters
into depth--indexed blocks, and let $f(\theta)$ denote the model output.
The \emph{full parameter Jacobian} is defined as
\[
J := \frac{\partial f}{\partial \theta},
\]
which we view as a linear operator from parameter space to function space.
Throughout this section, we adopt the block decomposition
\[
J = \bigl(J^{(1)}, J^{(2)}, \dots, J^{(L)}\bigr),
\qquad
J^{(\ell)} := \frac{\partial f}{\partial \theta_\ell},
\]
corresponding to concatenation over parameter blocks.

\paragraph{Layerwise parameter gradients and Jacobian factors.}
In residual architectures, the contribution $J^{(\ell)}$ of the $\ell$--th block
admits an explicit factorization reflecting the forward-- and backward--propagation
structure.
Let $\mathcal L$ be the training loss and denote the final hidden representation by $h_L$.
For each residual block, write the state-to-state Jacobian
\[
A_k := \frac{\partial h_{k+1}}{\partial h_k}.
\]
Then the layerwise parameter gradient admits the factorization
\[
J^{(\ell)}
:= \frac{\partial \mathcal L}{\partial \theta_\ell}
=
\Bigl(\frac{\partial \mathcal L}{\partial h_L}\Bigr)
\Bigl(\prod_{k=\ell+1}^{L-1} A_k\Bigr)
\frac{\partial F_\ell(h_\ell;\theta_\ell)}{\partial \theta_\ell},
\]
where $F_\ell$ denotes the residual branch of the $\ell$-th block.
Thus, while the \emph{full} Jacobian $J$ is obtained by concatenation,
each \emph{layerwise} Jacobian $J^{(\ell)}$ involves the same multiplicative
Jacobian structure analyzed in earlier residual learning arguments, with
left and right factors corresponding to downstream and upstream propagation,
respectively.

\paragraph{Residual form.}
In a residual learning configuration, the state Jacobian takes the near--identity form
\[
A_k = I + \varepsilon G_k,\qquad k=1,\dots,L-1,
\]
where $0<\varepsilon\le \varepsilon_0$ and $\|G_k\|_{\op}$ is uniformly controlled along the
optimization path.

\paragraph{Additive GRSD operator.}
Collect the layerwise contributions by concatenation
\[
J := \bigl(J^{(1)},\dots,J^{(L)}\bigr),
\]
and define the associated self--adjoint operator
\[
M^{(\ell)} := J^{(\ell)}J^{(\ell)*}.
\]
By concatenation, $M$ decomposes additively as
\[
M = \sum_{\ell=1}^{L} M^{(\ell)} = \sum_{\ell=1}^{L} J^{(\ell)} J^{(\ell)*}.
\]

i.e., $M$ is the sum of the self--adjoint contributions induced by individual
layerwise parameter Jacobians.

\subsection{Residual learning implies log--shift invariance}
\label{subsec:residual-theorem}

We now establish the main structural result of this section.
It shows that residual learning configurations induce log--shift invariance
through a two--stage mechanism:
first, beyond a finite depth threshold, individual layerwise Jacobian
contributions become depth--stationary on logarithmic spectral scales;
second, the additive structure of the GRSD operator causes finite--depth
inhomogeneities to be diluted once the total depth is sufficiently large.

We now state the main structural result of this section. It shows that residual learning configurations, once sufficiently deep, satisfy Condition~\ref{cond:log-shift-invariance}.

\begin{theorem}[Residual learning induces log--shift invariance beyond a depth threshold] \label{thm:bulk-stationary}
Consider a residual learning configuration in the sense of Section~\ref{subsec:residual-definition}. Fix an error tolerance $\eta \in (0,1)$. For any layer, if the layer depth $L$ satisfies 
\[ 
\ell \;\ge\; \frac{C_1}{\varepsilon^2 \eta^2} \ \ \vee\ \ \frac{C_2}{\varepsilon^2}\log\frac{1}{\eta}, 
\] 
where the constants $C_1, C_2$ depend only on the residual block law and not on $\ell$, then Condition~\ref{cond:log-shift-invariance} holds for $M^{(\ell)}$ on the intermediate spectral window. In particular, the renormalized log--bin shell coupling statistics depend, up to $O(\eta)$ error, only on relative log--scale separations. 

\end{theorem}

\begin{proof}[Proof sketch] The residual form $A_k = I + \varepsilon G_\ell$ induces additive log--spectral increments across depth. Under the absence of depth--wise parameter sharing and the controlled path regularity of Condition~\ref{cond:controlled-path}, these increments define a stationary Markov--additive process on spectral directions. Two mechanisms govern the depth threshold. First, Berry--Esseen type bounds \cite{bolthausen1982berryesseen} yield self--averaging of accumulated log--spectral increments once $L \gtrsim \varepsilon^{-2}\eta^{-2}$. Second, the near--identity residual perturbations induce quantitative mixing on the projective space, leading to isotropization of spectral directions on a scale $L \gtrsim \varepsilon^{-2}\log(1/\eta)$. Beyond the larger of these two thresholds, bin--averaged coupling statistics become translation--invariant in logarithmic scale. A complete proof is provided in Appendix~\ref{sec:appendix-residual-proof}. \end{proof} 

Theorem~\ref{thm:bulk-stationary} identifies residual learning as a mechanism
by which individual layerwise Jacobian contributions become asymptotically
free of absolute scale information once depth exceeds a finite threshold. Its proof relies in an essential way on the near--identity residual parameterization and depth--wise statistical stationarity. The emergence of directionally homogeneous behavior under repeated small random perturbations is consistent with classical results on random matrix products and Markovian mixing on projective spaces \citep{furstenberg1960products,levin2009markov}. As a result, the argument does not directly extend to architectures with explicit gating or depth--dependent weighting, nor to general non--skip architectures. This limitation should not be interpreted as ruling out scaling behavior in such models; rather, it indicates that different analytical tools would be required to verify Condition~\ref{cond:log-shift-invariance} outside the residual setting.

We now show that the additive structure of the GRSD operator converts this bulk stationarity into global log--shift invariance.

\begin{proposition}[Depth averaging yields Condition~\ref{cond:log-shift-invariance}]
\label{prop:depth-avg-logshift}
Assume the hypotheses of Theorem~\ref{thm:bulk-stationary}.  Fix a tolerance
$\eta\in(0,1)$ and an intermediate spectral window as in
Condition~\ref{cond:log-shift-invariance}.
Let $\ell_*$ be any depth index such that for every $\ell>\ell_*$,
the layerwise operator $M^{(\ell)}:=J^{(\ell)}J^{(\ell)*}$ satisfies
Condition~\ref{cond:log-shift-invariance} on the intermediate window
with error at most $O(\eta)$ (as guaranteed by Theorem~\ref{thm:bulk-stationary}).

Decompose
\[
M=\sum_{\ell=1}^{L}M^{(\ell)}
=\sum_{\ell=1}^{\ell_*}M^{(\ell)}+\sum_{\ell=\ell_*+1}^{L}M^{(\ell)}
=:M_{\mathrm{bd}}+M_{\mathrm{bulk}}.
\]
Then there exists a constant $C>0$, depending only on the uniform bounds in
Condition~\ref{cond:controlled-path}, such that if
\[
L \;\ge\; C\,\ell_*\,\eta^{-1},
\]
the renormalized log--bin shell coupling statistics induced by $M$
satisfy Condition~\ref{cond:log-shift-invariance} on the intermediate window
(with $O(\eta)$ error).
\end{proposition}

\begin{proof}[Proof sketch]
The key input is the additive structure of the GRSD operator across depth.
Since $M$ decomposes as a sum of layerwise contributions, the renormalized
log--bin shell coupling statistics considered in
Condition~\ref{cond:log-shift-invariance} decompose additively as well.
We therefore analyze the depth--averaged couplings by splitting $M$ into a
finite boundary part and a bulk part.

The boundary contribution involves only the first $\ell_*$ layers.
By the uniform operator bounds in
Condition~\ref{cond:controlled-path}, each layer contributes $O(1)$ to the
bin--level statistics on the intermediate spectral window.
After depth averaging, the total boundary contribution is therefore suppressed
by a factor $\ell_*/L$.

For the bulk layers, Theorem~\ref{thm:bulk-stationary} shows that each individual
layer already satisfies log--shift invariance on the intermediate window,
up to an $O(\eta)$ error.
Averaging these bulk contributions over depth preserves the same
log--shift invariant form, with the same $O(\eta)$ accuracy.

Combining the two parts, the depth--averaged shell coupling statistics induced by
$M$ consist of a log--shift invariant bulk term plus errors of order
$O(\eta)+O(\ell_*/L)$.
Choosing $L$ sufficiently large compared to $\ell_*$, specifically
$L\gtrsim \ell_*\eta^{-1}$, makes the boundary dilution error comparable to the
bulk error.
This yields Condition~\ref{cond:log-shift-invariance} for the full operator $M$
on the intermediate spectral window.
\end{proof}

Proposition~\ref{prop:depth-avg-logshift} completes the structural argument.
Log--shift invariance emerges not from exact scale symmetry at finite depth,
but from depth averaging over a stationary bulk of residual blocks, with
finite--depth inhomogeneities rendered negligible once the network exceeds
a controlled depth threshold.

\section{Interpretation of the Sufficient Conditions}
\label{sec:interpretation}

In this section we interpret the sufficient conditions introduced in Section~\ref{sec:sufficient-conditions} and clarify their conceptual meaning.
Our goal is not to justify these conditions empirically or to argue that they hold universally in deep learning, but rather to explain what structural and dynamical properties they encode.
Throughout, we emphasize that these conditions are sufficient but not necessary for power--law renormalizability of the GRSD shell dynamics.

\subsection{Boundedness of the gradient computation graph}
\label{sec:interp-locality}

Condition~\ref{cond:banded-jacobian} should be interpreted as a boundedness condition on the evolution of the gradient computation graph, rather than as a strict locality or acyclicity assumption.
It requires that the instantaneous evolution of each Jacobian block can be expressed using only a uniformly bounded neighborhood of other blocks, with coefficients that remain controlled along the training trajectory.

A canonical example of a bounded computation graph is a directed acyclic graph (DAG) \cite{abadi2016tensorflow,paszke2019pytorch}, such as those induced by standard feedforward networks, residual architectures, or transformers.
In these settings, gradient propagation is naturally confined by the depth of the graph, and the influence of any parameter update cannot spread arbitrarily far in an infinitesimal time step.
As a result, the induced Jacobian evolution is uniformly bounded in the sense required by Condition~\ref{cond:banded-jacobian}.

By contrast, recurrent architectures with unconstrained feedback loops may violate this boundedness condition.
In particular, when cyclic structures contribute divergent amplification—such as repeated multiplication by unstable recurrent weights—the Jacobian evolution may accumulate unbounded long-range couplings.
Classical recurrent neural networks (RNN) \cite{hopfield1982neural} and certain long short-term memory (LSTM) \cite{hochreiter1997long} configurations fall into this category when their recurrent dynamics is not properly controlled.
In such cases, the GRSD coarse-graining procedure may fail, as infinitesimal updates can induce global spectral reorganization.

Importantly, the presence of cycles alone does not preclude boundedness.
Recurrent or iterative architectures may still satisfy Condition~\ref{cond:banded-jacobian} provided that the contribution of cycles is contractive or effectively damped.
For example, when recurrent interactions decay exponentially
\cite{gu2021efficiently,smith2022simplified,gu2024mamba,peng2023rwkv,xu2025bringing},
the cumulative influence of long feedback paths remains finite, and Jacobian evolution can be bounded despite the existence of loops.
From the GRSD perspective, such architectures behave similarly to bounded-depth computation graphs at the level of infinitesimal learning dynamics.
We provide a precise proposition and a proof sketch for this case in Appendix \ref{prop:rwkv-banded}

Viewed in this light, Condition~\ref{cond:banded-jacobian} is best understood as excluding learning configurations in which gradient propagation is unbounded or unstable, rather than as forbidding recurrence or complex connectivity per se.
It delineates the class of computation graphs for which shell-wise coarse-graining remains well-defined over the training horizon.

\subsection{Initial functional incoherence}
\label{sec:interp-incoherence}

Condition~\ref{cond:init-incoherence} constrains the geometry of the Jacobian at initialization.
It requires that correlations between distant Jacobian blocks, measured through operator inner products, decay sufficiently fast so that their tail is summable.
Modern zero-mean random i.i.d. initializations
\cite{glorot2010understanding,he2015delving,klambauer2017self}
commonly used in deep learning naturally satisfy the condition of initial gradient incoherence, in the sense that gradient components associated with distinct functional modes are uncorrelated at initialization.

This assumption does not require strict independence or orthogonality between blocks.
Local correlations are permitted, and even expected, as a consequence of shared inputs, architectural coupling, or structured parameterization.
The essential requirement is that long-range correlations are weak enough to prevent the formation of coherent global modes at initialization.

Within the GRSD framework, initial functional incoherence plays a role analogous to short-range correlations in statistical physics.
It ensures that coarse-grained shell dynamics is not dominated by fine-tuned cancellations or global alignments present at initialization.
As shown in the proof of Theorem~\ref{thm:power-law-renormalizability}, this condition propagates along the training trajectory under the locality and stability assumptions.

\subsection{Controlled Jacobian evolution}
\label{sec:interp-stability}

Condition~\ref{cond:controlled-path} imposes a stability constraint on the learning trajectory.
It requires that both the Jacobian and its instantaneous time derivative remain uniformly bounded over the training horizon.

Conceptually, this condition ensures that learning proceeds through a sequence of controlled, incremental updates rather than through abrupt transitions.
From the perspective of GRSD, it guarantees that temporal evolution and spectral coarse-graining are compatible: the shell dynamics remains well-defined and does not develop singular behavior over finite time.

This assumption should be viewed as an abstraction of common stability-inducing practices in modern deep learning, such as normalization, residual connections, and conservative optimization schedules
\citep{terjek2022framework,pennington2017resurrecting,tarnowski2019dynamical}. 
However, the condition itself is agnostic to the specific optimization algorithm and does not rely on stochasticity or noise-induced regularization.

\subsection{Log--shift invariance and rigidity of renormalized dynamics}
\label{sec:interp-scale-cov}

Condition~\ref{cond:log-shift-invariance} encodes the absence of any
intrinsic absolute scale in the renormalized shell couplings.
Unlike pointwise equivariance or parameter--level symmetries, this
condition is formulated at the level of coarse--grained spectral
statistics: it requires that, on the intermediate spectral window,
renormalized shell interactions depend only on relative separations in
logarithmic scale and not on absolute spectral position.

Crucially, log--shift invariance should be understood as a structural
constraint rather than a generic consequence of stability or locality.
It rules out not only explicit scale parameters, but also implicit
time--dependent reference scales that could emerge during training and
induce drift of the effective shell dynamics.
In this sense, Condition~\ref{cond:log-shift-invariance} is a genuinely
nontrivial requirement that restricts the admissible large--scale
behavior of the renormalized dynamics.

Within the GRSD framework, log--shift invariance alone enforces
homogeneity of the renormalized shell equations, but does not by itself
fix the functional form of the shell velocity.
The emergence of a power--law velocity is instead a rigidity
phenomenon: when log--shift invariance is combined with the intrinsic
time--rescaling covariance of gradient flow, all admissible functional
forms except pure power laws are excluded.
Such rigidity phenomena, in which scale invariance combined with
consistency constraints restrict admissible large-scale behavior to
power laws, are well known in renormalization-group analyses of critical
phenomena and turbulent transport
\citep{kadanoff1966scaling,wilson1983renormalization,forster1977large}.
If Condition~\ref{cond:log-shift-invariance} fails, the shell dynamics
may exhibit scale--dependent drift, moving spectral reference points, or
broken scaling behavior, even if the remaining conditions are satisfied.

\subsection{Summary: renormalizable shell dynamics}
\label{sec:interp-summary}

Taken together, the four conditions introduced in
Section~\ref{sec:sufficient-conditions} enforce a learning dynamics that
is local in the computation graph, weakly correlated at initialization,
stable along the optimization path, and free of intrinsic absolute scales
in the renormalized spectral description.
These properties are precisely those required for the GRSD shell
dynamics to admit a well--defined and renormalizable closure under
logarithmic spectral coarse--graining.

We emphasize that none of the conditions alone is sufficient to
guarantee power--law behavior.
Conditions~\ref{cond:banded-jacobian}--\ref{cond:controlled-path} ensure the existence and stability of the renormalized
shell dynamics, while Condition~\ref{cond:log-shift-invariance} restricts its large--scale structure
by enforcing log--shift invariance.
Power--law scaling then emerges as a rigidity consequence when these
structural properties are combined with the intrinsic time--rescaling
covariance of gradient flow.
Under this combination, the GRSD framework predicts that the only
admissible large--scale form of the shell velocity is a power law.

\section{Implications for Modern Architectures}
\label{sec:architectures}

The sufficient conditions introduced in
Section~\ref{sec:sufficient-conditions} are formulated at the level of
operator dynamics and gradient evolution, rather than in terms of
specific model classes.
As a result, they can be meaningfully interpreted across a broad range
of modern architectures.
At the same time, the conditions—especially
Condition~\ref{cond:log-shift-invariance}—are structural and nontrivial,
and are not generically guaranteed by architectural design alone.
In this section, we discuss how common deep learning architectures align
with the sufficient conditions, and where their limitations arise.
Throughout, we emphasize that the discussion is interpretive rather than
universal: the presence of architectural features consistent with some
conditions does not imply that power--law scaling must occur, but rather
that it is not a priori excluded within the GRSD framework.

\subsection{Multilayer perceptrons}
\label{sec:arch-mlp}

Multilayer perceptrons (MLPs) \cite{rosenblatt1957perceptron} provide the simplest setting in which the
sufficient conditions can be examined.
Their feedforward structure naturally induces locality in the gradient
computation graph, consistent with
Condition~\ref{cond:banded-jacobian}.
Moreover, standard random initialization schemes typically lead to weak
functional correlations between distant layers in sufficiently wide
networks, aligning with
Condition~\ref{cond:init-incoherence} in early training regimes.

When trained with stable optimization procedures, MLPs may also satisfy
Condition~\ref{cond:controlled-path} over substantial training horizons.
However, the absence of an explicit near--identity accumulation
mechanism means that Condition~\ref{cond:log-shift-invariance} is not
generically enforced in MLPs.
As a result, while MLPs often admit a renormalizable shell description,
power--law behavior cannot be structurally guaranteed and may depend
sensitively on training dynamics, data distribution, or implicit
regularization effects.

\subsection{Convolutional networks}
\label{sec:arch-cnn}

Convolutional networks introduce additional structure through spatial
locality and weight sharing.
These features further reinforce locality in the gradient computation
graph and typically support
Condition~\ref{cond:banded-jacobian}.
Standard initialization and normalization practices can also promote
weak long--range functional correlations at initialization and stable
Jacobian evolution during training.

From the perspective of GRSD, convolutional architectures often admit a
well--defined renormalized shell dynamics.
However, as with MLPs, log--shift invariance of renormalized shell
couplings is not generically guaranteed.
Architectural choices such as aggressive pooling, explicit bottlenecks,
or strongly scale--dependent preprocessing may introduce effective
spectral reference scales, leading to deviations from simple power--law
behavior even when the remaining conditions are approximately satisfied.

An important exception arises in convolutional architectures equipped
with residual parameterizations.
Residual convolutional networks \citep{he2016resnet} introduce an explicit near--identity
accumulation mechanism across depth, which aligns closely with the
structural setting analyzed in Section~\ref{sec:residual-condition}.
In this case, log--shift invariance of the renormalized shell couplings
can be established under additional assumptions, rather than merely
postulated.
The residual parameterizations, as introduced in ResNet architectures
\citep{he2016resnet}, have also been shown to promote stable Jacobian
spectra across depth \citep{tarnowski2019dynamical}.

\subsection{Transformer architectures}
\label{sec:arch-transformer}

Transformer architectures combine multiple computational modules within
each block and rely heavily on normalization and residual connections.
These design choices promote stable training dynamics and help maintain
bounded Jacobian norms over long training horizons, consistent with
Condition~\ref{cond:controlled-path}
\citep{vaswani2017attention,ba2016layernorm}. 

The effective gradient propagation distance per training step in
transformers remains limited, supporting the locality requirement of
Condition~\ref{cond:banded-jacobian}.
At initialization, widely used random parameterizations lead to weak
functional correlations between distant blocks in sufficiently wide
settings.
Nevertheless, the presence of residual connections alone does not
automatically imply log--shift invariance of renormalized shell
couplings.
Whether Condition~\ref{cond:log-shift-invariance} is realized depends on
additional structural and dynamical properties, such as the near--identity
nature of residual accumulation and the absence of depth--dependent
gating or weighting effects.

Accordingly, while transformers empirically exhibit robust scaling
behavior across many regimes, deviations from power--law scaling are
naturally interpreted within GRSD as violations of one or more
sufficient conditions, rather than as contradictions of the framework.

\subsection{Structured state--space models}
\label{sec:arch-ssm}

Structured state--space models and related architectures introduce \cite{gu2021efficiently,smith2022simplified,gu2024mamba}
explicit recurrence or long--range temporal structure.
From the perspective of
Condition~\ref{cond:banded-jacobian}, such models are compatible with the
sufficient conditions provided that recurrent interactions remain
effectively contractive, so that Jacobian evolution does not induce
unbounded long--range coupling.

When recurrence is stabilized through normalization, decay mechanisms,
or controlled parameterization, the resulting Jacobian dynamics may
remain localized and amenable to renormalized shell descriptions.
However, strong recurrence or unstable feedback can violate locality or
stability assumptions, leading to non--renormalizable shell dynamics and
the breakdown of simple scaling behavior.
In such cases, GRSD predicts deviations from power--law scaling as a
natural consequence of the learning configuration.

\subsection{Summary}
\label{sec:arch-summary}

Across architectures, the sufficient conditions identified in this work
do not single out a specific model class.
Rather, they characterize a regime of learning dynamics in which
locality, weak long--range correlations, stability, and the absence of
intrinsic spectral reference scales coexist.
While modern architectures frequently incorporate design choices that
support some of these properties, none guarantees them in isolation.

From the GRSD perspective, the empirical success of neural scaling laws
reflects the prevalence of learning configurations that lie within this
renormalizable regime.
Equally, observed deviations from power--law behavior signal departures
from the sufficient conditions—particularly from
Condition~\ref{cond:log-shift-invariance}—rather than failures of the
GRSD framework itself.

\section{Experimental Validation of Condition 4}
\label{sec:experiments}

The theory developed in this work relies on three intertwined empirical
assumptions about the spectral organization of learning dynamics:

\begin{enumerate}
  \item \textbf{Controlled inter-shell coupling.}
  Spectral shells interact weakly, so that cross-shell transport acts as a
  subleading and stable perturbation.

  \item \textbf{Necessity of shell coarse-graining.}
  Intra-shell mixing is strong compared to inter-shell transport, making
  coarse-graining over spectral shells essential for any effective description.

  \item \textbf{Realization of Condition~\ref{cond:log-shift-invariance} by residual learning.}
  In residual architectures, the induced spectral transport operator becomes
  approximately Toeplitz in an intermediate spectral window, reflecting
  approximate shift-invariance across shells.
\end{enumerate}

In this section, we provide minimal yet direct empirical evidence for all three
claims using controlled experiments on CIFAR-10 \cite{krizhevsky2014cifar}.
We emphasize that the goal is not to exhaustively characterize training
dynamics, but to validate the structural assumptions required for a
closed, shell-based description of learning.

\subsection{Experimental setup}
\label{sec:exp_setup}

We evaluate the spectral coupling structure induced by training dynamics on
image classification using CIFAR-10.
All experiments are conducted with ResNet-18 and a corresponding plain
(non-residual) architecture with identical depth and channel configuration.

\paragraph{Models.}
The residual model is a standard ResNet-18 \cite{he2016resnet}.
The plain counterpart (PlainNet-18) is constructed by removing all residual
skip connections while keeping convolutional blocks, batch normalization,
and parameter counts identical.
Both models share the same trained parameter checkpoints.

\paragraph{Dataset and checkpoints.}
We use the CIFAR-10 validation set with batch size $B=64$.
Unless otherwise stated, all spectral operators are evaluated at a fixed
training checkpoint using a local finite-difference approximation of the
operator derivative (see below), ensuring that both architectures are compared
at the same effective training stage.

\paragraph{Error representation.}
The error vector $e(x;\theta)$ is constructed from model logits using the
$p-\text{onehot}$ formulation,
\[
e = p_\theta(x) - \mathrm{onehot}(y),
\]
flattened over batch and class dimensions.
This choice aligns the error space with the Fisher-like geometry of
classification loss and is used consistently across all experiments.

\subsection{Spectral operator estimation}
\label{sec:exp_operator}

Our goal is to probe the spectral structure of the operator
\[
M(\theta) = \mathbb{E}_x \big[ J(x;\theta) J(x;\theta)^\top \big],
\]
where $J = \partial e / \partial \theta$ is the Jacobian of the error vector
with respect to model parameters.

\paragraph{Monte Carlo operator evaluation.}
The expectation defining $M$ is approximated using a small number ($S=4$)
of fixed mini-batches sampled from the validation set.
Matrix--vector products $Mv$ are computed without explicitly forming $M$,
using a Jacobian--vector product followed by a vector--Jacobian product
(JVP--VJP), enabling scalable operator access.

\paragraph{Block Lanczos projection.}
We apply a block Lanczos algorithm \cite{golub2010matrices,ubaru2017fast} with block size $b=8$ and $k=16$ steps,
resulting in a Krylov subspace of dimension $r = bk = 128$.
This produces a block-tridiagonal matrix $T$ whose eigenpairs
$\{(\lambda_i, u_i)\}$ approximate the spectrum of $M$ in the probed subspace.

\paragraph{Directional derivative of the operator.}
To probe spectral transport, we compute a finite-difference approximation
of the operator derivative
\[
\dot M \approx \frac{M(\theta - \varepsilon \nabla \mathcal L) - M(\theta)}{\varepsilon},
\]
where $\varepsilon$ is a small step along the gradient direction.
The projected operator $\Omega = U^\top (Q^\top \dot M Q) U$ is expressed in
the Ritz eigenbasis of $M$.

More details regarding the estimation algorithm are included in Appendix \ref{app:lanczos}.

\subsection{Shell binning and Toeplitz diagnostics}
\label{sec:exp_toeplitz}

Eigenvalues of $M$ are binned into logarithmically spaced spectral shells
based on $\log \lambda$.
Given the projected operator $\Omega$, we define a coarse-grained shell
coupling matrix $K$ by averaging squared couplings within each shell pair:
\[
K_{ij} = \mathbb{E}_{u \in \mathcal S_i, v \in \mathcal S_j}
\big[ |\Omega_{uv}|^2 \big].
\]

\paragraph{Toeplitz residual.}
To quantify the degree of shift-invariance across shells, we compare $K$
to its best Toeplitz approximation obtained by averaging along diagonals.
The normalized Frobenius residual
\[
\mathrm{Res}(K)
= \frac{\|K - \mathcal T(K)\|_F^2}{\|K\|_F^2}
\]
serves as our primary diagnostic: lower values indicate stronger Toeplitz
structure i.e. the coupling strength relies on the relative scale shift rather than absolute scales.

All scores are computed over sliding spectral windows.
We focus on windows centered in the intermediate spectral regime, where the
assumptions underlying our effective theory are expected to hold.

\subsection{Diagonal scaling of spectral couplings}
\label{sec:diag_scaling}

\begin{figure}[t]
    \centering
    \begin{subfigure}[t]{0.48\linewidth}
        \centering
        \includegraphics[width=\linewidth]{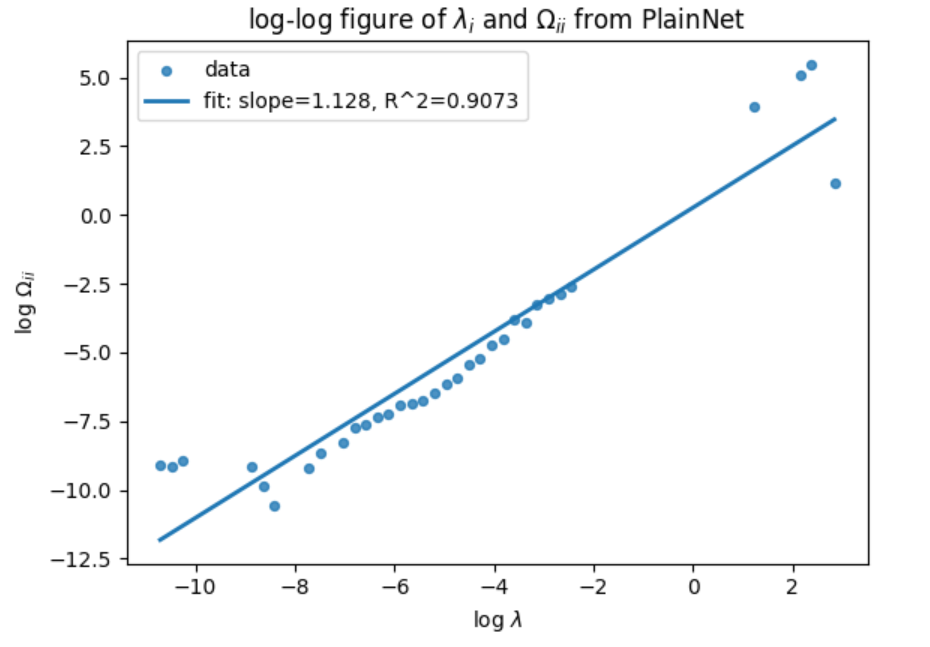}
        \caption{PlainNet}
        \label{fig:plainnet_diag}
    \end{subfigure}
    \hfill
    \begin{subfigure}[t]{0.48\linewidth}
        \centering
        \includegraphics[width=\linewidth]{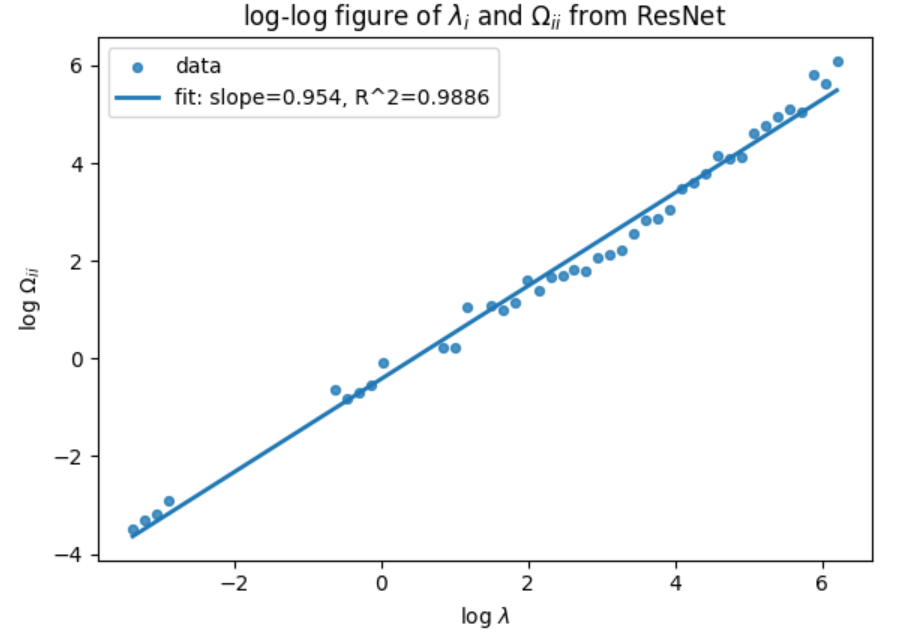}
        \caption{ResNet}
        \label{fig:resnet_diag}
    \end{subfigure}
    \caption{
    Log--log plots of diagonal spectral couplings $\Omega_{ii}$ versus eigenvalues
    $\lambda_i$ for PlainNet (left) and ResNet (right).
    Each point corresponds to a log-shell averaged diagonal entry.
    Solid lines denote linear least-squares fits in logarithmic coordinates.
    ResNet exhibits a scaling exponent closer to unity and significantly reduced
    scatter compared to PlainNet, consistent with the scale-free diagonal
    structure predicted by Condition~\ref{cond:log-shift-invariance}.
    }
    \label{fig:diag_scaling}
\end{figure}

We empirically examine the diagonal components of the renormalized spectral
coupling matrix, $\Omega_{ii}$, to test the scale-free diagonal structure
postulated in Condition~\ref{cond:log-shift-invariance}. Recall that Condition~\ref{cond:log-shift-invariance} predicts that, on an
intermediate spectral window, the diagonal terms satisfy a linear scaling
relation
\begin{equation}
\Omega_{ii} \propto \lambda_i,
\end{equation}
corresponding to scale-invariant drift in logarithmic spectral coordinates.

Figure~\ref{fig:diag_scaling} shows log--log scatter plots of $\Omega_{ii}$ versus
$\lambda_i$ for both architectures, together with linear least-squares fits in
logarithmic coordinates. In both cases, the data exhibit a clear power-law
relation over a broad intermediate spectral range, consistent with the
scale-free diagonal ansatz of Condition~\ref{cond:log-shift-invariance}.

Quantitatively, the ResNet model displays an almost perfectly linear scaling
with fitted slope close to unity and an excellent coefficient of determination
($R^2 \approx 0.99$). In contrast, while PlainNet also shows approximate
power-law behavior, its fitted slope deviates more noticeably from unity and
exhibits increased scatter, reflected in a lower $R^2$ value. This indicates
that the diagonal scaling relation is significantly more stable and coherent in
the presence of residual connections.

\paragraph{Interpretation.}
The near-linear relation $\Omega_{ii} \sim \lambda_i$ observed in ResNet
supports the assumption that diagonal spectral couplings introduce no intrinsic
absolute scale beyond that already encoded in $\lambda_i$. This behavior is a
necessary ingredient for scale-invariant log-spectral drift and underpins the
emergence of renormalizable spectral dynamics. The degradation of this relation
in PlainNet suggests that residual connections play a structural role in
stabilizing diagonal scaling, thereby promoting the conditions required for
log-shift invariance and controlled spectral transport.

Overall, these results provide direct empirical support for the diagonal
component of Condition~\ref{cond:log-shift-invariance} and highlight the architectural dependence of
scale-free spectral structure.

\begin{figure}[t]
  \centering
  \includegraphics[width=0.8\linewidth]{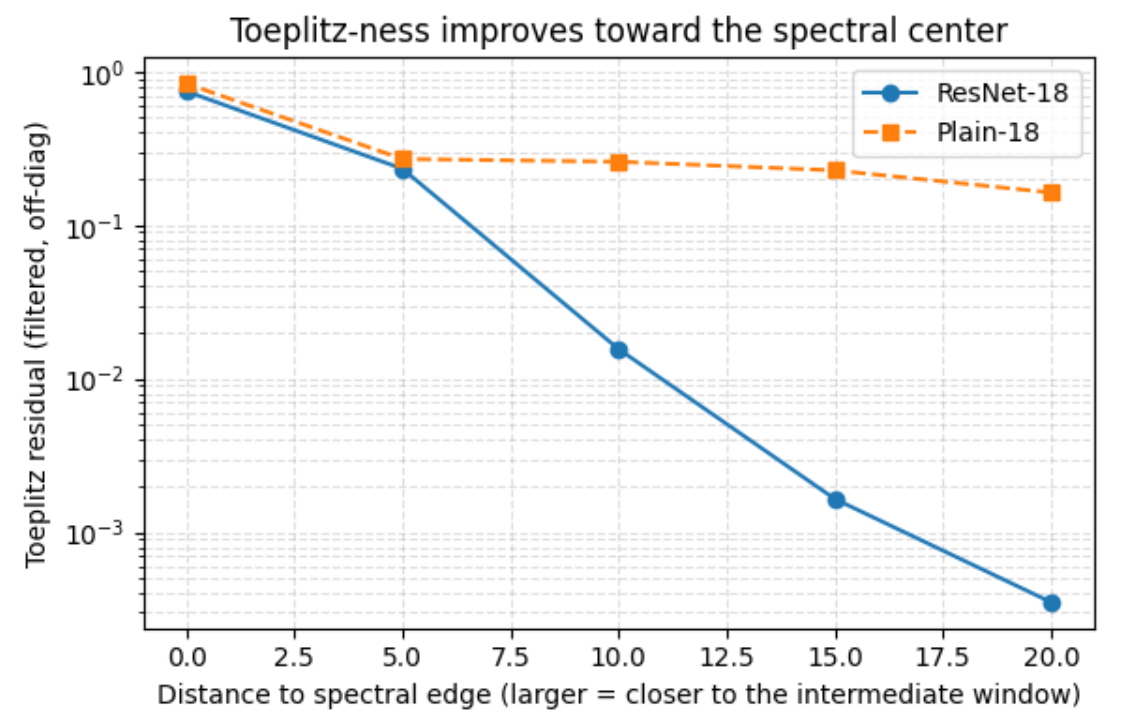}
  \caption{
  Toeplitz residual of the coarse-grained spectral coupling matrix as a function
  of distance to the spectral edge (larger values correspond to windows closer
  to the intermediate spectral regime).
  Residual networks exhibit a rapid and monotonic improvement in Toeplitz-ness
  toward the spectral center, while the corresponding plain networks show
  substantially weaker and flatter behavior.
  }
  \label{fig:toeplitz_resnet_vs_plain}
\end{figure}

\subsection{Residual learning induces Toeplitz spectral structure}
\label{sec:exp_residual}

Figure~\ref{fig:toeplitz_resnet_vs_plain} compares the Toeplitz residual of
ResNet-18 and PlainNet-18 as the analysis window moves from the spectral edge
toward the spectral center.

Residual networks exhibit a rapid and monotonic improvement in Toeplitz
structure when approaching the intermediate spectral regime.
In contrast, the plain architecture shows significantly weaker Toeplitz
behavior with little sensitivity to window position.
This provides direct empirical evidence that residual learning facilitates
the approximate shift-invariance across spectral shells required by
Condition~\ref{cond:log-shift-invariance}.

Near the spectral edges, Toeplitz structure deteriorates for both models.
This behavior is expected, as these regions lie outside the effective
spectral and time window assumed by the theory, rather than indicating a
failure of the condition itself.






\section{Discussion and Limitations}
\label{sec:discussion}

This work identifies a set of sufficient conditions under which the GRSD shell dynamics admits a power--law renormalized description.
While these conditions provide a rigorous route to power--law behavior, they are intentionally not presented as universal or exhaustive.
In this section we discuss the scope of the results, their limitations, and directions for future investigation.

\subsection{Non-universal and broken scaling regimes}
\label{sec:broken-scaling}

Empirical studies have repeatedly shown that neural scaling laws are not universal.
Depending on architecture, data distribution, optimization stability, or training regime, power--law behavior may weaken, fragment into multiple regimes, or fail altogether.
Within the GRSD framework, such deviations are naturally interpreted as violations of one or more sufficient conditions identified in this paper.

For example, strong long-range functional correlations at initialization may obstruct coarse-grained shell closure.
Unstable optimization trajectories may generate abrupt spectral reorganizations that invalidate controlled Jacobian evolution.
Similarly, the presence of intrinsic scales in the data or preprocessing pipeline may break statistical scale covariance.
Rather than viewing these failures as anomalies, GRSD predicts them as signatures of non-renormalizable learning dynamics.

\subsection{Sufficiency versus necessity}
\label{sec:sufficiency-vs-necessity}

A central limitation of the present analysis is that the conditions identified here are sufficient but not necessary.
We do not claim that all learning configurations exhibiting power--law scaling must satisfy these assumptions, nor do we rule out alternative mechanisms leading to scaling behavior.

In particular, stochastic optimization effects, noise-induced regularization, or implicit averaging over training trajectories may produce effective renormalization even when deterministic stability or scale covariance assumptions fail.
Characterizing such mechanisms lies beyond the scope of the current work.
Our results should therefore be understood as delineating a class of learning dynamics for which power--law renormalizability can be rigorously established, rather than as an exhaustive theory of all observed scaling laws.

\subsection{Relation to kernel and feature--learning limits}
\label{sec:relation-to-ntk}

The GRSD framework encompasses both kernel-like and feature-learning regimes, depending on the structure of the Jacobian evolution.
In kernel limits, where $M(t)$ remains close to its initialization, shell dynamics may become effectively frozen, leading to degenerate or trivial velocity fields.
In contrast, feature-learning regimes permit nontrivial spectral transport and admit renormalized dynamics.

The sufficient conditions identified here are compatible with both perspectives, but they do not reduce to either limit.
Instead, they describe a regime in which learning dynamics is sufficiently structured to admit coarse-grained closure, yet sufficiently flexible to allow nontrivial spectral flow.
Understanding how these conditions interpolate between known limits remains an open direction for future work.

\subsection{Future directions}
\label{sec:future-work}

Several directions for further investigation emerge from this work.
First, it would be valuable to develop empirical diagnostics that test the sufficient conditions directly, for example by measuring locality of Jacobian evolution or scale covariance of gradient statistics during training.
Second, extending the analysis to stochastic optimization dynamics may clarify whether noise can relax or replace certain stability assumptions.
Finally, exploring non-renormalizable regimes within GRSD may shed light on multi-scaling phenomena, regime transitions, and the limits of predictability in deep learning.

More broadly, this work suggests that neural scaling laws are best understood not as universal empirical facts, but as manifestations of renormalizable learning dynamics.
GRSD provides a framework in which both the success and the failure of scaling laws can be interpreted within a unified spectral and operator-theoretic perspective.

\bibliography{references}
\bibliographystyle{plain}

\appendix
\section{Technical proofs}
\label{sec:appendix-proofs}

\subsection{Boundedness of Particular Recurrent Networks}
\label{prop:rwkv_SSM}
\begin{proposition}[RWKV/SSM implies an (effective) graph--banded Jacobian path]
\label{prop:rwkv-banded}
Consider a sequence model whose forward dynamics admits a linear state--space (SSM) form
\begin{equation}
\label{eq:ssm-forward}
h_{t}=A(\theta,t)\,h_{t-1}+B(\theta,t)\,u_t,\qquad
y_t=C(\theta,t)\,h_t,
\end{equation}
where $u_t$ is the input at time $t$, and the training loss is $L=\sum_{t=1}^T \ell(y_t)$.
Let $J(t)$ denote the function--space Jacobian along training time $t$ (as in Theorem~\ref{thm:power-law-renormalizability}),
and decompose it into computation--graph blocks $J=(J^{(1)},\dots,J^{(L)})$.

Assume the following \emph{stable--propagator} conditions hold uniformly along training time $t\in[0,T]$:
\begin{enumerate}
\item \textbf{Uniform exponential stability.}
There exist constants $\rho\in(0,1)$ and $C_A<\infty$ such that for all integers $k\ge0$,
\begin{equation}
\label{eq:exp-stability}
\sup_{t\in[0,T]}\ \sup_{\tau}\ \bigl\|A(\theta(t),\tau)\,A(\theta(t),\tau-1)\cdots A(\theta(t),\tau-k+1)\bigr\|_{\op}
\le C_A \rho^{k}.
\end{equation}
\item \textbf{Uniform boundedness of local Jacobians.}
The operator norms of the local derivatives $\partial_\theta A,\partial_\theta B,\partial_\theta C$ and the backprop factors
$\partial_{y_t}\ell$ are uniformly bounded on $[0,T]$.
\end{enumerate}

Then for any tolerance $\delta\in(0,1)$ there exists an \emph{effective interaction range}
\begin{equation}
\label{eq:effective-range}
K_\delta \;=\; O\!\left(\frac{\log(1/\delta)}{1-\rho}\right)
\end{equation}
such that the GRSD Jacobian path is $\delta$--approximately graph--banded in the sense that
\begin{equation}
\label{eq:banded-approx}
\dot J^{(l)}(t) \in \mathrm{span}\bigl\{J^{(m)}(t): |m-l|\le K_\delta\bigr\}\;+\;\mathcal{R}^{(l)}_\delta(t),
\qquad
\|\mathcal{R}^{(l)}_\delta(t)\|_{\op}\le \delta,
\end{equation}
uniformly for all $t\in[0,T]$ and all blocks $l$.
In particular, when the blocks $J^{(l)}$ are chosen as coarse--grained time/graph blocks with width $\gtrsim K_\delta$,
Condition~(1) in Theorem~\ref{thm:power-law-renormalizability} holds with some constant $K=O(1)$ at that coarse scale.
\end{proposition}

\begin{proof}[Proof sketch]
The key point is that in a stable SSM/RWKV--type model, every ``loop'' contribution along the recurrent edge carries a factor of the
state propagator, which decays exponentially in the number of traversals.

Fix a parameter block index $l$ and consider the block Jacobian $J^{(l)}$ that collects the function--space derivatives contributed
by a localized portion of the computation graph (e.g.\ a window of time steps, or a local module in a scan/SSM implementation).
Differentiating $J^{(l)}$ with respect to training time produces terms of the form
\[
\dot J^{(l)} \;\sim\; \sum_{\text{paths}} (\text{backprop factors})\cdot (\partial_\theta A,\partial_\theta B,\partial_\theta C)\cdot
(\text{state propagators}).
\]
Whenever a term couples block $l$ to a far block $m$ (large $|m-l|$), it must traverse a long recurrent path in the unrolled graph.
By \eqref{eq:exp-stability}, the operator norm of the corresponding propagator is at most $C_A\rho^{|m-l|}$ up to bounded local factors.
Thus the total contribution from distances $|m-l|>K$ is bounded by a geometric tail:
\[
\sum_{|m-l|>K} O(\rho^{|m-l|}) \;\le\; O(\rho^{K}).
\]
Choosing $K=K_\delta$ so that $\rho^{K_\delta}\lesssim \delta$ yields \eqref{eq:banded-approx}.
Finally, if we define the GRSD blocks at a coarse scale larger than $K_\delta$, the residual $\mathcal{R}^{(l)}_\delta$ becomes
negligible at the GRSD shell scale, and the exact banded form holds at that renormalized resolution.
\end{proof}

\subsection{Proof of Theorem \ref{thm:power-law-renormalizability}}
\label{app:proof_power_law}
\begin{theorem}[Power--law renormalizability of GRSD shell dynamics]
Fix a \emph{learning configuration} (architecture, initialization, and optimization path) and consider the GRSD shell dynamics defined on logarithmic spectral shells (i.e., shells in $s=\log\lambda$ as in GRSD).
Let $J(t)$ denote the (function--space) Jacobian, decomposed into computation--graph blocks
$J(t)=(J^{(1)}(t),\dots,J^{(L)}(t))$, and let $M(t):=J(t)J(t)^*$.

Assume the following conditions hold on a training horizon $t\in[0,T]$:
\begin{enumerate}
\item \textbf{Graph--banded Jacobian path.}
There exists a constant $K=O(1)$ such that for every block index $l$,
\begin{equation}
\label{eq:assump-banded-Jdot_app}
\dot J^{(l)}(t)\in \mathrm{span}\bigl\{J^{(m)}(t): |m-l|\le K\bigr\}.
\end{equation}

\item \textbf{Initial functional incoherence (summable tail).}
There exists a nonnegative sequence $\{\varepsilon_k\}_{k\ge 1}$ with $\sum_{k\ge1}\varepsilon_k<\infty$ such that
\begin{equation}
\label{eq:assump-init-incoh_app}
\|J^{(l)}(0)^*J^{(m)}(0)\|_{\mathrm{op}}\le \varepsilon_{|l-m|}\qquad \forall l,m.
\end{equation}

\item \textbf{Controlled Jacobian path.}
There exists $C_J<\infty$ such that
\begin{equation}
\label{eq:assump-controlled-path_app}
\sup_{t\in[0,T]}\Bigl(\|J(t)\|_{\mathrm{op}}+\|\dot J(t)\|_{\mathrm{op}}\Bigr)\le C_J.
\end{equation}

\item \textbf{Log--shift invariance of renormalized shell couplings.}
On an intermediate spectral window, the renormalized shell coupling statistics are translation invariant in the logarithmic coordinate
$s=\log\lambda$, i.e. 
there exists a kernel $K_h(\Delta)$, $\Delta=(j-i)h$, such that
\[
\widehat{\mathsf K}_{ij}
\;=\;
K_h\!\bigl((j-i)h\bigr)
\;+\;
\mathrm{err}(n,h,L),
\]
where $\mathrm{err}(n,h,L)\to 0$ in the joint limit of large width $n$, small bin size $h$,
and sufficient depth $L$.
\end{enumerate}

In addition, assume the GRSD structural property already adopted in the main text:
(i) \emph{intra--shell couplings are antisymmetric}, so shell--internal transfers cancel in the shell energy balance.

Then the induced GRSD shell dynamics admits a renormalized velocity field
\begin{equation}
\label{eq:power-law-v_app}
v(\lambda,t)=c(t)\lambda^{a}
\end{equation}
for some exponent $a\in\mathbb R$ and scalar coefficient $c(t)$.
\end{theorem}

\begin{proof}
We prove Theorem~\ref{thm:power-law-renormalizability} by decomposing the argument into three steps.
Step~I establishes an effective one--dimensional shell conservation law with \emph{nearest--neighbor boundary fluxes}, where the closure is obtained as a structural consequence of $M=JJ^*$ via a Taylor expansion.
Steps~II and~III use Condition~\ref{cond:log-shift-invariance} together with gradient--flow covariance to derive the
power--law constraint on the velocity field.

\paragraph{Notation.}
Let $\{S_\alpha\}$ be the GRSD logarithmic shells in $\lambda$ (equivalently, equal--width bins in $s=\log\lambda$).
Let $E_\alpha(t)$ denote the shell energy (shell--aggregated error energy), and let $\varepsilon(\lambda,t)$ be the corresponding shell--averaged energy density
(defined by the usual piecewise--constant interpolation over shells).
Let $F_{\alpha+\frac12}(t)$ denote the net flux across the boundary between $S_\alpha$ and $S_{\alpha+1}$, and let $D_\alpha(t)$ be the shell--aggregated dissipation term.

\paragraph{Step I: From (1)--(3) to shell conservation with nearest--neighbor boundary fluxes.}

\emph{(I.a) Propagation of summable incoherence along the Jacobian path.}
Define the block--correlation matrix
\[
C_{lm}(t):=J^{(l)}(t)^*J^{(m)}(t).
\]
Differentiate $C_{lm}(t)$:
\begin{equation}
\label{eq:Clm-derivative_app}
\dot C_{lm}(t)=\dot J^{(l)}(t)^*J^{(m)}(t)+J^{(l)}(t)^*\dot J^{(m)}(t).
\end{equation}
By \eqref{eq:assump-banded-Jdot_app}, for each $l$ there exist operator coefficients $A_{lp}(t)$ such that
\begin{equation}
\label{eq:Jdot-linear-comb_app}
\dot J^{(l)}(t)=\sum_{|p-l|\le K} A_{lp}(t)\,J^{(p)}(t).
\end{equation}
By \eqref{eq:assump-controlled-path_app}, these coefficients are uniformly bounded: there exists $C_A<\infty$ such that
\begin{equation}
\label{eq:A-bound_app}
\sup_{t\in[0,T]}\max_{l}\sum_{|p-l|\le K}\|A_{lp}(t)\|_{\mathrm{op}}\le C_A.
\end{equation}
Substituting \eqref{eq:Jdot-linear-comb_app} into \eqref{eq:Clm-derivative_app} yields
\begin{equation}
\label{eq:Clm-rec_app}
\dot C_{lm}(t)
=
\sum_{|p-l|\le K} C_{pm}(t)\,A_{lp}(t)^*
+
\sum_{|q-m|\le K} A_{mq}(t)\,C_{lq}(t).
\end{equation}
Taking operator norms and using submultiplicativity gives
\begin{equation}
\label{eq:Clm-norm-ineq_app}
\frac{d}{dt}\|C_{lm}(t)\|_{\mathrm{op}}
\;\le\;
\sum_{|p-l|\le K}\|C_{pm}(t)\|_{\mathrm{op}}\|A_{lp}(t)\|_{\mathrm{op}}
+
\sum_{|q-m|\le K}\|A_{mq}(t)\|_{\mathrm{op}}\|C_{lq}(t)\|_{\mathrm{op}}.
\end{equation}

Define the distance--profile envelope
\[
u_k(t):=\sup_{|l-m|=k}\|C_{lm}(t)\|_{\mathrm{op}},\qquad k\ge 0.
\]
Using \eqref{eq:Clm-norm-ineq_app} and that $|p-m|$ differs from $|l-m|$ by at most $K$, one obtains
\begin{equation}
\label{eq:uk-ineq_app}
\dot u_k(t)\;\le\; 2C_A \sum_{j=-K}^{K} u_{k+j}(t),
\end{equation}
with $u_{k}=0$ for $k<0$.
Let $U(t):=\sum_{k\ge 0} \rho^k u_k(t)$ for some $\rho\in(0,1)$.
Then \eqref{eq:uk-ineq_app} implies a Gr\"onwall inequality
\begin{equation}
\label{eq:U-ineq_app}
\dot U(t)\;\le\; C_\rho\,U(t),
\qquad
C_\rho:=2C_A\Bigl(\sum_{j=-K}^{K}\rho^{-j}\Bigr)<\infty,
\end{equation}
hence $U(t)\le e^{C_\rho t}U(0)$.
By \eqref{eq:assump-init-incoh_app}, $U(0)<\infty$ since $\sum_k\varepsilon_k<\infty$.
Therefore $U(t)<\infty$ for all $t\in[0,T]$, and in particular there exists a summable tail $\{\tilde\varepsilon_k\}_{k\ge1}$ such that
\begin{equation}
\label{eq:propagated-incoh_app}
\|J^{(l)}(t)^*J^{(m)}(t)\|_{\mathrm{op}}
\le \tilde\varepsilon_{|l-m|},
\qquad
\sum_{k\ge1}\tilde\varepsilon_k<\infty,
\qquad \forall t\in[0,T].
\end{equation}

\emph{(I.b) Automatic nearest--neighbor closure via a parameter--space Taylor expansion and the exact $\Omega$ formula.}
By antisymmetry, intra--shell transfers cancel in the energy balance, hence
\begin{equation}
\label{eq:shell-balance-preflux_app_v3}
\frac{d}{dt}E_\alpha(t)
=
\sum_{\beta\neq\alpha} T_{\alpha\beta}(t) - D_\alpha(t),
\qquad
\sum_{\alpha}\sum_{\beta\neq\alpha}T_{\alpha\beta}(t)=0,
\end{equation}
where $T_{\alpha\beta}(t)$ denotes net transfer from shell $\beta$ to shell $\alpha$.

We now show that \emph{direct transfers between non--adjacent shells are second order}, and therefore the inter--shell dynamics closes (up to a controlled remainder) on nearest--neighbor boundary fluxes\footnote{The idea is inspired by \cite{tian2025provablescalinglawsfeature}}.
The argument uses a Taylor expansion of the model output in parameter space and the exact formula for the spectral rotation generator $\Omega$.

\paragraph{Taylor expansion of the output and the order of $\Delta M$.}
Let $f(\theta)$ denote the model output in function space, with Jacobian $J(\theta)=\partial_\theta f(\theta)$ and squared loss $\mathcal L(\theta)=\tfrac12\|e(\theta)\|^2$ where $e(\theta)=f(\theta)-y$.
For a small parameter increment $\delta\theta$, the output admits the expansion
\begin{equation}
\label{eq:f-taylor_app_v3}
f(\theta+\delta\theta)
=
f(\theta)+J(\theta)\,\delta\theta+\frac12\,\mathcal H(\theta)[\delta\theta,\delta\theta]+O(\|\delta\theta\|^3),
\end{equation}
where $\mathcal H(\theta)$ is the second derivative (a Hessian tensor) of $f$.

Consider one infinitesimal gradient--flow step $\delta\theta = \dot\theta(t)\,\delta t$ with $\dot\theta(t)=-\nabla_\theta \mathcal L(\theta(t))=-J(\theta(t))^*e(t)$.
The \emph{linearized} (first--order) model $f_{\mathrm{lin}}(\theta+\delta\theta):=f(\theta)+J(\theta)\delta\theta$ keeps $J(\theta)$ \emph{fixed}, hence the corresponding operator
\[
M_{\mathrm{lin}} := J(\theta)J(\theta)^*
\]
is constant over the step. Therefore any change of $M(t)=J(t)J(t)^*$ must come from the quadratic and higher terms in \eqref{eq:f-taylor_app_v3}, i.e. from curvature.
Equivalently, over a time increment $\delta t$,
\begin{equation}
\label{eq:M-increment-order_app_v3}
\Delta M(t)
:=
M(t+\delta t)-M(t)
=
O(\|\delta\theta\|)
=
O(\delta t),
\end{equation}
and the contribution of $\Delta M$ to the \emph{error update} appears only at second order in time,
because the first--order error update uses $M(t)$ frozen:
\begin{equation}
\label{eq:e-taylor-time_app_v3}
e(t+\delta t)
=
e(t) - M(t)e(t)\,\delta t
\;+\;
\frac{(\delta t)^2}{2}\Bigl(M(t)^2e(t)-\dot M(t)e(t)\Bigr)
\;+\;O((\delta t)^3).
\end{equation}
In particular, the term involving $\dot M$ (hence curvature) enters the error evolution only through the $(\delta t)^2$ term.

\paragraph{Exact $\Omega$ formula and gap suppression across non--adjacent shells.}
Let $\{(\lambda_i(t),u_i(t))\}$ be an orthonormal eigen-decomposition of $M(t)$ on the GRSD window and define the instantaneous rotation generator
\[
\Omega_{ij}(t):=\langle u_i(t),\dot u_j(t)\rangle,\qquad \Omega_{ij}=-\Omega_{ji}.
\]
For $i\neq j$, the standard exact identity gives
\begin{equation}
\label{eq:omega-formula_app_v3}
\Omega_{ij}(t)
=
\frac{\langle u_i(t),\dot M(t)u_j(t)\rangle}{\lambda_j(t)-\lambda_i(t)}.
\end{equation}
Assume the GRSD shell partition $\{S_\alpha\}$ is chosen so that there exists a minimal spectral gap between \emph{non--adjacent} shells:
\begin{equation}
\label{eq:gap-nonadj_app_v3}
\inf_{t\in[0,T]}\ \inf_{\substack{i\in S_\alpha,\,j\in S_\beta\\ |\alpha-\beta|\ge2}}
|\lambda_j(t)-\lambda_i(t)| \ge \Delta.
\end{equation}
Under the controlled path condition \eqref{eq:assump-controlled-path_app} we have $\sup_{t\in[0,T]}\|\dot M(t)\|_{\mathrm{op}}<\infty$, hence \eqref{eq:omega-formula_app_v3} and \eqref{eq:gap-nonadj_app_v3} imply the uniform bound
\begin{equation}
\label{eq:omega-gap-bound_app_v3}
\sup_{t\in[0,T]}\ \sup_{\substack{i\in S_\alpha,\,j\in S_\beta\\ |\alpha-\beta|\ge2}}
|\Omega_{ij}(t)|
\ \le\
\frac{\sup_{t\in[0,T]}\|\dot M(t)\|_{\mathrm{op}}}{\Delta}.
\end{equation}

\paragraph{Second--order non--adjacent transfers and boundary--flux closure.}
Write the error in the instantaneous eigenbasis, $e(t)=\sum_i a_i(t)u_i(t)$, and note that inter--shell exchange is induced by basis rotation.
Combining the time--Taylor expansion \eqref{eq:e-taylor-time_app_v3} with the fact that $\Omega$ is generated by $\dot M$ through \eqref{eq:omega-formula_app_v3}, we obtain the key structural consequence:
\emph{since $\dot M$ enters $e(t+\delta t)$ only at order $(\delta t)^2$, any energy exchange mediated by $\Omega$ is at least second order in $\delta t$.}
Moreover, for non--adjacent shells $|\alpha-\beta|\ge2$, the gap bound \eqref{eq:gap-nonadj_app_v3} controls the corresponding rotation coefficients via \eqref{eq:omega-gap-bound_app_v3}, so the associated transfers are uniformly of second order:
\begin{equation}
\label{eq:Tab-second-order_app_v3}
T_{\alpha\beta}(t)
=
O((\delta t)^2)\cdot O\!\Bigl(\Delta^{-1}\sup_{t\in[0,T]}\|\dot M(t)\|_{\mathrm{op}}\Bigr),
\qquad |\alpha-\beta|\ge2,
\end{equation}
with constants controlled by the bounds in Conditions (1)--(3).
(Adjacent shells are not covered by the uniform gap bound and therefore may contribute at leading order; these contributions define the renormalized boundary fluxes.)

Consequently, the inter--shell balance \eqref{eq:shell-balance-preflux_app_v3} closes, up to a controlled remainder collecting the non--adjacent second--order transfers, in telescoping boundary--flux form
\begin{equation}
\label{eq:shell-balance-flux_app_v3}
\frac{d}{dt}E_\alpha(t)
=
F_{\alpha-\frac12}(t)-F_{\alpha+\frac12}(t)-D_\alpha(t)+R_\alpha(t),
\end{equation}
where $F_{\alpha+\frac12}(t)$ depends only on spectral content in a fixed neighborhood of the boundary between $S_\alpha$ and $S_{\alpha+1}$, and the remainder
$R_\alpha(t):=\sum_{|\beta-\alpha|\ge2}T_{\alpha\beta}(t)$ is uniformly controlled on $[0,T]$ by the second--order estimate \eqref{eq:Tab-second-order_app_v3} together with the summable incoherence propagation from Step~I.a.
Passing to the shell--averaged density representation yields the effective conservation law
\begin{equation}
\label{eq:conservation-law_app}
\partial_t \varepsilon(\lambda,t) + \partial_\lambda J(\lambda,t) = -D(\lambda,t) + r(\lambda,t),
\end{equation}
where $r$ corresponds to the controlled remainder.
This completes Step~I.

\paragraph{Step II: Scale covariance as a rigidity consequence of log--shift invariance.}

We now invoke Condition~\ref{cond:log-shift-invariance} (log--shift invariance).
On the intermediate spectral window, the renormalized shell coupling statistics governing the effective flux law \eqref{eq:conservation-law_app} are asymptotically translation invariant in $s=\log\lambda$.
Equivalently, for any $\tau\in\mathbb{R}$, a shift $\varepsilon(s,t)\mapsto \varepsilon(s+\tau,t)$ induces a corresponding shift of the flux $J(s,t)$, without introducing any intrinsic length scale.

As a consequence, the velocity field $v(s,t):=J(s,t)/\varepsilon(s,t)$ cannot depend on the absolute position $s$ except through multiplicative rescaling.
In particular, any admissible dependence of $v$ on $s$ must be compatible with a functional relation of the form
\[
v(s+\tau,t) = \alpha(\tau)\,v(s,t_\tau),
\]
for some scale factor $\alpha(\tau)$ and a possibly rescaled time $t_\tau$.
To identify the form of $\alpha(\tau)$ and the relation between $t_\tau$ and $t$, we use the covariance of gradient flow.

\paragraph{Step III: Gradient--flow covariance and exclusion of time--dependent log shifts.}

We first record a basic covariance property of gradient flow.

\begin{lemma}[Gradient-flow covariance fixes the time dependence of $v$]
\label{lem:gf-covariance}
Consider gradient-flow training with mean-squared error loss
\[
\mathcal L(\theta)=\tfrac12\|e(\theta)\|^2,
\qquad
\dot\theta(t)=-\nabla_\theta \mathcal L(\theta(t)),
\]
and induced error dynamics
\[
\dot e(t)=-M(t)e(t),
\qquad
M(t)=J_{\theta(t)}J_{\theta(t)}^* .
\]
Let $\lambda(t)$ denote an eigenvalue of $M(t)$ and define the spectral drift
\[
v(\lambda,t):=\frac{d\lambda}{dt}.
\]

Then $v$ satisfies the covariance functional equation
\begin{equation}
v(a\lambda,t/a)=a^2\,v(\lambda,t),
\qquad \forall a>0,
\label{eq:covariance_v}
\end{equation}
and consequently admits the representation
\begin{equation}
v(\lambda,t)=t^{-2}F(\lambda t),
\qquad
F(u):=v(u,1),
\qquad t>0.
\label{eq:v_scaling_form}
\end{equation}
\end{lemma}

\begin{proof}
\
\textbf{Step 1: Loss rescaling induces time reparameterization.}

For any $a>0$, consider the rescaled loss $\mathcal L_a:=a\,\mathcal L$.
The corresponding gradient flow satisfies
\[
\dot\theta_a(t)=-\nabla_\theta \mathcal L_a(\theta_a(t))
=-a\,\nabla_\theta \mathcal L(\theta_a(t)).
\]
If $\theta(t)$ solves the original gradient flow, then
$\theta_a(t):=\theta(a t)$ satisfies
\[
\dot\theta_a(t)=a\,\dot\theta(a t)
=-a\,\nabla_\theta \mathcal L(\theta(a t))
=-a\,\nabla_\theta \mathcal L(\theta_a(t)),
\]
and hence solves the rescaled flow. Thus, rescaling the loss by $a$
is equivalent to the time reparameterization $t\mapsto a t$.

\textbf{Step 2: Loss rescaling multiplies the error operator.}

For MSE loss,
$\nabla_\theta \mathcal L(\theta)=-J_\theta^* e$, so
$\dot\theta=J_\theta^* e$.
Differentiating $e=y-f_\theta$ yields
\[
\dot e=-J_\theta \dot\theta=-J_\theta J_\theta^* e=-M(t)e.
\]
Under $\mathcal L_a=a\mathcal L$, the parameter flow becomes
$\dot\theta_a=aJ_\theta^* e$, and the induced error dynamics are
\[
\dot e_a=-J_\theta \dot\theta_a=-a\,J_\theta J_\theta^* e_a=-a\,M(t)e_a.
\]
Hence, at the level of the error ODE, loss rescaling induces
\[
M(t)\mapsto a\,M(t),
\quad\text{and therefore}\quad
\lambda(t)\mapsto a\,\lambda(t).
\]

\textbf{Step 3: Covariance of eigenvalue trajectories.}

Combining Steps 1--2, the eigenvalue trajectories obey
\[
\lambda_a(t)=a\,\lambda(a t),
\]
where $\lambda_a$ denotes the eigenvalue under $\mathcal L_a$.
Differentiating with respect to $t$ gives
\[
v_a(t)=\dot\lambda_a(t)
=\frac{d}{dt}\big(a\,\lambda(a t)\big)
=a^2\,\dot\lambda(a t)
=a^2\,v(a t).
\]
Rewriting this relation in field form yields
\eqref{eq:covariance_v}.

\textbf{Step 4: Solving the functional equation.}

Fix $t>0$ and choose $a=t$ in \eqref{eq:covariance_v}. Then
\[
v(t\lambda,1)=t^2\,v(\lambda,t),
\]
which can be rearranged as
\[
v(\lambda,t)=t^{-2}v(t\lambda,1).
\]
Defining $F(u):=v(u,1)$ and substituting $u=\lambda t$ yields
\eqref{eq:v_scaling_form}.
\end{proof}

Condition~\ref{cond:log-shift-invariance} rules out that, on an intermediate spectral window, shell couplings are translation--invariant in
$s=\log\lambda$. In other words, it forces that $\dot s = \frac{\dot\lambda}{\lambda}$ is invariant to the absolute scale of $\lambda$.

\paragraph{Scale-invariance of log-spectral drift induced by Condition~\ref{cond:log-shift-invariance}.}
\label{app:scale_invariant_sdot}

Recall that under Condition~\ref{cond:log-shift-invariance}, the shell-level coupling statistics admit a
log-shift invariant form on an intermediate spectral window,
\begin{equation}
\hat\Omega_{ij}(t)=
\begin{cases}
K_h\big((j-i)h\big), & i\neq j,\\[4pt]
c\,\lambda_i, & i=j,
\end{cases}
\label{eq:Omega_log_shift_app}
\end{equation}
where $h$ denotes the log-shell width, $\lambda_i=\lambda_0 e^{-ih}$, and
$K_h$ depends only on the relative log-scale separation.

The evolution of individual eigenvalues takes the form
\begin{equation}
\dot{\lambda}_i
=
\sum_{j\neq i}
\frac{\hat\Omega_{ij}}{\lambda_i-\lambda_j}
+
\hat\Omega_{ii}.
\end{equation}
Introducing the log-spectral coordinate $s_i:=\log\lambda_i$, we obtain
\begin{equation}
\dot s_i
=
\frac{\dot\lambda_i}{\lambda_i}
=
\sum_{j\neq i}
\frac{\hat\Omega_{ij}}{\lambda_i(\lambda_i-\lambda_j)}
+
\frac{\hat\Omega_{ii}}{\lambda_i}.
\label{eq:sdot_def_app}
\end{equation}
The diagonal contribution simplifies immediately as
\begin{equation}
\frac{\hat\Omega_{ii}}{\lambda_i}=c,
\end{equation}
which is independent of the absolute scale of $\lambda_i$.

For the off-diagonal terms, writing
\(
\lambda_i-\lambda_j=\lambda_i\big(1-e^{-(s_j-s_i)}\big)
\)
and using the log-shift invariant structure
\(
\hat\Omega_{ij}=K_h(s_j-s_i),
\)
we find that each summand depends only on the relative separation $s_j-s_i$.
Upon coarse-graining over shells and passing to the continuum limit, the sum
reduces to an integral of the form
\begin{equation}
\dot s(s,t)
=
\int \mathrm{d}s'\,
\frac{K(s'-s)}{1-e^{-(s'-s)}}
+
c,
\label{eq:sdot_scale_free_app}
\end{equation}
where $K$ is a scale-free kernel inherited from $K_h$.

Equation~\eqref{eq:sdot_scale_free_app} depends exclusively on relative
log-scale separations and contains no reference to the absolute value of $s$
or $\lambda$. Consequently,
\begin{equation}
\dot s(s+\Delta,t)=\dot s(s,t),
\qquad \forall\,\Delta\in\mathbb{R},
\end{equation}
demonstrating that the log-eigenvalue drift is invariant under global spectral
rescalings $\lambda\mapsto e^{\Delta}\lambda$. This establishes explicitly that
Condition~\ref{cond:log-shift-invariance} enforces the absence of any intrinsic absolute spectral scale in
the renormalized spectral dynamics.

As a consequence, any effective drift velocity in \(s\)-space cannot depend explicitly on
the absolute position \(s\). In particular,
\[
u(s,t):=\dot s=\frac{1}{\lambda}\dot\lambda=\frac{v(\lambda,t)}{\lambda}=c(t)
\]
is independent of \(\lambda\) within the window.

Equivalently,
\[
\boxed{
\;
v(\lambda,t)=\lambda\,c(t)
\quad\text{on the intermediate spectral window.}
\;}
\tag{C2}
\]
Combining this formula with $
v(\lambda,t)=t^{-2}v(t\lambda,1).
$, we have:
\[
\boxed{
\;
v(\lambda,t)=c(t_0)\frac{\lambda}{t}\
\quad\text{on the intermediate spectral window $t\in[t_0, T]$.}
\;}
\tag{C2}
\]



This establishes the power--law form of the velocity field and completes
the proof of Theorem~\ref{thm:power-law-renormalizability}.
\end{proof}

\subsection{Proof of Theorem~\ref{thm:bulk-stationary}}
\label{sec:appendix-residual-proof}

In this appendix we provide a complete proof of
Theorem~\ref{thm:bulk-stationary}.
The argument refines the intermediate result
(\emph{log--bin relative--scale kernel})
by making explicit the probabilistic and dynamical mechanisms underlying
log--shift invariance.
Compared to the informal derivation sketched in the main text,
the present proof addresses three technical issues:
(i) the additive log--spectral increments are treated as a
Markov--additive process rather than i.i.d.;
(ii) directional mixing is quantified via a uniform mixing time;
(iii) no independence is assumed between eigenvector couplings and
spectral gaps.

\subsubsection{Setup and notation}

Let
\[
J^{(L)} := J_L J_{L-1}\cdots J_1,
\qquad
J_\ell = I + \varepsilon G_\ell ,
\]
where $\varepsilon>0$ is sufficiently small and $\{G_\ell\}$ satisfy the
structural assumptions of Section~\ref{subsec:residual-definition}.
Define
\[
M^{(L)} := J^{(L)} (J^{(L)})^\ast ,
\]
with eigen-decomposition
\[
M^{(L)} = \sum_{u=1}^n \lambda_u \,\phi_u \phi_u^\ast ,
\qquad
s_u := \log \lambda_u .
\]

Let $\{\mathcal B_i\}$ denote logarithmic spectral bins of width $h$ in $s$,
and write $I_i := \{u : s_u \in \mathcal B_i\}$.
We restrict attention to an intermediate spectral window bounded away
from the spectral edges, as specified in
Condition~\ref{cond:controlled-path}.

\subsubsection{Markov--additive structure of log--spectral increments}
\label{subsec:markov}

For any unit vector $u\in\mathbb S^{n-1}$ define the normalized direction
process
\[
u_\ell := \frac{J_\ell u_{\ell-1}}{\|J_\ell u_{\ell-1}\|},
\qquad u_0=u.
\]
The associated single--layer log--increment is
\[
\delta_\ell(u_{\ell-1})
:= \log \|J_\ell u_{\ell-1}\|
= \log \|(I+\varepsilon G_\ell)u_{\ell-1}\|.
\]

The pair $(u_\ell,\delta_\ell)$ defines a Markov--additive process \cite{meyn2012markov} on
$\mathbb S^{n-1}\times\mathbb R$.
Under Condition~\ref{cond:controlled-path} and the residual small--step
assumption, $\delta_\ell$ admits uniform moment bounds
\[
\mathbb E\bigl[|\delta_\ell|^2\bigr]=O(\varepsilon^2),
\qquad
\mathbb E\bigl[|\delta_\ell|^3\bigr]=O(\varepsilon^3),
\]
uniformly in $u_{\ell-1}$.
Moreover, on the intermediate spectral window the conditional variance
$\mathrm{Var}(\delta_\ell\mid u_{\ell-1})$ is bounded below by
$c_0\varepsilon^2$ for some $c_0>0$.

\subsubsection{Self--averaging of accumulated log--increments}
\label{subsec:self-average}

Let
\[
S_L := \sum_{\ell=1}^L \bigl(\delta_\ell - \mu\bigr),
\qquad
\mu := \mathbb E[\delta_\ell],
\]
where expectation is taken with respect to the stationary distribution
of the direction process.

Since $u_\ell$ is a uniformly ergodic Markov chain
(see Section~\ref{subsec:mixing}),
standard Berry--Esseen bounds \cite{bolthausen1982berryesseen} for Markov--additive processes apply.
In particular, there exists $C_{\mathrm{BE}}<\infty$ such that
\[
\sup_x
\Bigl|
\mathbb P\!\left(
\frac{S_L}{\sigma\sqrt L}\le x
\right)-\Phi(x)
\Bigr|
\;\le\;
\frac{C_{\mathrm{BE}}}{\sqrt L},
\]
where $\sigma^2=\mathrm{Var}(\delta_\ell)=\Theta(\varepsilon^2)$.

Consequently, for any Lipschitz test function $f$ and any tolerance
$\eta>0$, there exists
\[
L \;\ge\; C_1(\varepsilon\eta)^{-2}
\]
such that the distribution of accumulated log--increments is
$\eta$--close (in total variation or Wasserstein distance) to a
translation--invariant limit on the intermediate spectral window.
This yields approximate shift--invariance of log--spectral statistics.

\subsubsection{Directional mixing}
\label{subsec:mixing}

The direction chain $\{u_\ell\}$ evolves on the unit sphere.
Under the residual perturbation $J_\ell=I+\varepsilon G_\ell$ with
approximately isotropic law and no parameter sharing across depth,
the chain admits the uniform distribution as its unique invariant
measure.

For products of i.i.d. random matrices, classical results of Furstenberg and Kesten establish the existence of a unique stationary distribution on the projective space and ergodicity of the induced Markov dynamics, leading to direction independent asymptotic behavior of matrix products under suitable conditions \citep{furstenberg1960products}. The directional updates induced by the random residual perturbations define a Markov process on the unit sphere, whose mixing properties can be analyzed in the classical Markov chain framework \citep{levin2009markov}. Standard coupling or diffusion--approximation arguments imply that the
mixing time $\tau_{\mathrm{mix}}(\eta)$ satisfies
\[
\tau_{\mathrm{mix}}(\eta)
\;\le\;
C_2\,\varepsilon^{-2}\log\frac{1}{\eta}.
\]
Hence, for all $L\ge\tau_{\mathrm{mix}}(\eta)$, the law of $u_L$ is
$\eta$--close to uniform, uniformly over initial directions.
This ensures isotropization of eigenvector statistics on the
intermediate spectral window. 
\subsubsection{Log--bin averaging and kernel emergence}

\paragraph{Off-Diagonal shell statistics}Let $X_{uv}:=\phi_u^\ast A\phi_v$ denote generic quadratic couplings
entering the GRSD closure.
By Condition~\ref{cond:controlled-path}, conditional second moments
satisfy \cite{koltchinskii2017normal}
\[
\mathbb E[X_{uv}^2\mid s_u,s_v]
=
\frac{1}{n}a(s_u,s_v)+O(n^{-3/2}),
\]
with $a$ Lipschitz on the intermediate window.

Combining the shift--invariance from Appendix~\ref{subsec:self-average} with directional mixing
from Appendix~\ref{subsec:mixing} yields
\[
a(s_u,s_v)=b(s_v-s_u)+O(\eta),
\]
uniformly on the window.
Logarithmic bin averaging then gives
\[
\mathbb E[\widehat X_{ij}]
=
\frac{1}{n}K_h((j-i)h)
+O\!\left(\frac{\eta}{n}\right),
\]
for some kernel $K_h$ depending only on the bin width $h$.

For quantities involving spectral gaps, such as
$\Omega_{v\to u}=X_{uv}/(\lambda_v-\lambda_u)$,
we restrict to pairs with $|s_v-s_u|\ge\delta$.
Since $\lambda_v-\lambda_u$ is then uniformly bounded away from zero on
the intermediate window, Cauchy--Schwarz yields the same bin--level
structure without invoking independence.

\paragraph{On-Diagonal shell statistics}

The diagonal structure in Condition~\ref{cond:log-shift-invariance} can be obtained from the same bulk
self-mixing mechanism used in the off-diagonal analysis.
Fix an intermediate spectral window and condition on a mode $u$ with
$\lambda=\langle u,Mu\rangle$ lying in a log-shell $B_i$.
Beyond a boundary depth, residual blocks induce fast mixing on the projective
space of directions, so that the conditional distribution of mode directions
$u \mid (\lambda\in B_i)$ is asymptotically isotropic (up to a small mixing error).

Consequently, for any layerwise diagonal quadratic observable of the form
$\Omega^{(\ell)}_{uu}=\langle u,\dot M^{(\ell)}u\rangle$, isotropy implies that its
conditional expectation retains only the \emph{Rayleigh component}:
\[
\mathbb{E}\!\left[\Omega^{(\ell)}_{uu}\,\middle|\,\lambda\in B_i\right]
=\alpha_\ell(t)\,
\mathbb{E}\!\left[\langle u,Mu\rangle\,\middle|\,\lambda\in B_i\right]
\;+\;O(\eta)\,\mathbb{E}\!\left[\lambda\,\middle|\,\lambda\in B_i\right],
\]
for some scalar $\alpha_\ell(t)$ independent of the absolute shell index $i$ on the
intermediate window, and where the $O(\eta)$ term captures the finite-depth
imperfect mixing (the same $O(\eta)$ level as in the bulk part of the Toeplitz
argument). This reduction to the Rayleigh component is standard for isotropic quadratic
observables and follows from classical results on quadratic forms under
rotational invariance \cite{horn2013matrix,anderson2010random,koltchinskii2017normal}.

Summing over layers and applying the same boundary--bulk decomposition as in
depth dilution, the boundary contribution is suppressed by $\ell^*/L$, while the
bulk average preserves the shell-independent coefficient. Hence the full diagonal
coupling satisfies, on the intermediate window,
\[
\mathbb{E}\!\left[\Omega_{uu}\,\middle|\,\lambda\in B_i\right]
=\big(c(t)+O(\eta)+O(\ell^*/L)\big)\,
\mathbb{E}\!\left[\lambda\,\middle|\,\lambda\in B_i\right],
\]
i.e.\ $\Omega_{ii}=c(t)\lambda_i+\mathrm{err}_{ii}$ after shell averaging, which is the
on-diagonal part of Condition~\ref{cond:log-shift-invariance}.
Standard self-averaging over depth and within-shell averaging further reduce the
residual fluctuations around this mean.

\subsubsection{Conclusion}

Combining Appendix~\ref{subsec:self-average} and~\ref{subsec:mixing}, log--shift invariance holds whenever
\[
L
\;\ge\;
\frac{C_1}{\varepsilon^2\eta^2}
\;\vee\;
\frac{C_2}{\varepsilon^2}\log\frac{1}{\eta},
\]
which is precisely the depth threshold stated in
Theorem~\ref{thm:bulk-stationary}.
This completes the proof.
\qed

\subsection{Proof of Proposition~\ref{prop:depth-avg-logshift}}

In this appendix we provide a complete proof of
Proposition~\ref{prop:depth-avg-logshift}. The proof proceeds by exploiting the additive structure of the GRSD operator
across depth.  Since the renormalized shell coupling statistics used in
Condition~\ref{cond:log-shift-invariance} are quadratic and binwise in the
Jacobian, they depend on the network only through the operator
$M=JJ^\ast$ and therefore decompose additively over layerwise contributions.
This allows us to split $M$ into a finite boundary part and a bulk part,
control the boundary contribution using uniform operator bounds, and transfer
the layerwise log--shift invariance established in
Theorem~\ref{thm:bulk-stationary} to the depth--averaged operator.
The resulting argument shows that sufficiently large depth suppresses
nonstationary boundary effects and upgrades bulk stationarity to global
log--shift invariance.

\begin{proof}
Throughout, we work on the intermediate spectral window of
Condition~\ref{cond:log-shift-invariance}.
Let $\widehat{\mathsf K}_{ij}(\cdot)$ denote the renormalized log--bin shell
coupling statistics appearing in Condition~\ref{cond:log-shift-invariance}.
In GRSD these couplings are quadratic statistics of the Jacobian and therefore
depend on $J$ only through $M=JJ^\ast$; in particular, under the additive
decomposition $M=\sum_{\ell}M^{(\ell)}$ (Section~\ref{subsec:residual-definition},
paragraph ``Additive GRSD operator''), the corresponding bin--level couplings
decompose additively:
\begin{equation}
\label{eq:K-additive}
\widehat{\mathsf K}_{ij}(M)
=
\sum_{\ell=1}^{L}\widehat{\mathsf K}_{ij}\!\bigl(M^{(\ell)}\bigr).
\end{equation}
(Concretely, this is because all GRSD closure statistics are formed from
shell/boundary quadratic forms of $M$, hence linear in $M$ once the binning is fixed.)

Define the depth--averaged couplings
\[
\overline{\mathsf K}_{ij}(M):=\frac{1}{L}\,\widehat{\mathsf K}_{ij}(M).
\]
Using \eqref{eq:K-additive} and splitting into boundary and bulk,
\[
\overline{\mathsf K}_{ij}(M)
=
\frac{1}{L}\sum_{\ell=1}^{\ell_*}\widehat{\mathsf K}_{ij}\!\bigl(M^{(\ell)}\bigr)
\;+\;
\frac{1}{L}\sum_{\ell=\ell_*+1}^{L}\widehat{\mathsf K}_{ij}\!\bigl(M^{(\ell)}\bigr)
=: \overline{\mathsf K}_{ij}^{\rm bd}+\overline{\mathsf K}_{ij}^{\rm bulk}.
\]

\paragraph{Step 1: the boundary contribution is negligible once $L\gg \ell_*$.}
By Condition~\ref{cond:controlled-path},
\[
\sup_{t\in[0,T]}\|J(t)\|_{\op}\le C_J
\quad\Longrightarrow\quad
\sup_{\ell\le L}\|M^{(\ell)}\|_{\op}
=
\sup_{\ell\le L}\|J^{(\ell)}J^{(\ell)\ast}\|_{\op}
\le
\sup_{\ell\le L}\|J^{(\ell)}\|_{\op}^2
\le C_J^2.
\]
Moreover, the bin--level couplings $\widehat{\mathsf K}_{ij}(\cdot)$ are bounded
quadratic statistics on the intermediate window (again part of the statistical
regularity in Condition~\ref{cond:controlled-path}), hence there exists a constant
$C_{\mathsf K}$ depending only on $C_J$ such that
\[
\sup_{i,j}\sup_{\ell\le L}\bigl|\widehat{\mathsf K}_{ij}(M^{(\ell)})\bigr|\le C_{\mathsf K}.
\]
Therefore
\begin{equation}
\label{eq:bd-term}
\sup_{i,j}\bigl|\overline{\mathsf K}_{ij}^{\rm bd}\bigr|
\le
\frac{1}{L}\sum_{\ell=1}^{\ell_*}\sup_{i,j}\bigl|\widehat{\mathsf K}_{ij}(M^{(\ell)})\bigr|
\le
C_{\mathsf K}\,\frac{\ell_*}{L}.
\end{equation}

\paragraph{Step 2: the bulk contribution inherits log--shift invariance from Theorem~\ref{thm:bulk-stationary}.}
By choice of $\ell_*$ and Theorem~\ref{thm:bulk-stationary}, for every $\ell>\ell_*$,
the layerwise couplings satisfy
\[
\widehat{\mathsf K}_{ij}(M^{(\ell)})
=
K_h\bigl((j-i)h\bigr) + O(\eta),
\]
uniformly over bins $(i,j)$ in the intermediate window, for the same bin width $h$
as in Condition~\ref{cond:log-shift-invariance}.
Averaging over $\ell=\ell_*+1,\dots,L$ yields
\begin{equation}
\label{eq:bulk-term}
\overline{\mathsf K}_{ij}^{\rm bulk}
=
\frac{L-\ell_*}{L}\,K_h\bigl((j-i)h\bigr)
\;+\; O(\eta),
\end{equation}
where the $O(\eta)$ term is uniform over $(i,j)$ in the window.
(Here we used that the average of $O(\eta)$ errors remains $O(\eta)$.)

\paragraph{Step 3: conclude Condition~\ref{cond:log-shift-invariance} for $M$.}
Combining \eqref{eq:bd-term}--\eqref{eq:bulk-term} gives
\[
\overline{\mathsf K}_{ij}(M)
=
\frac{L-\ell_*}{L}\,K_h\bigl((j-i)h\bigr)
\;+\; O(\eta)\;+\;O\!\left(\frac{\ell_*}{L}\right).
\]
Since $\frac{L-\ell_*}{L}=1-O(\ell_*/L)$, the prefactor only perturbs the kernel
by an additional $O(\ell_*/L)$ on the same window, and thus
\[
\overline{\mathsf K}_{ij}(M)
=
K_h\bigl((j-i)h\bigr) + O(\eta)+O\!\left(\frac{\ell_*}{L}\right).
\]
Choosing $L\ge C\,\ell_*\eta^{-1}$ makes the boundary dilution term
$O(\ell_*/L)$ at most $O(\eta)$, so the total error is $O(\eta)$.
This is exactly Condition~\ref{cond:log-shift-invariance} for the couplings induced by $M$
(on the intermediate spectral window), completing the proof.
\end{proof}

\section{Spectral operator estimation via block Lanczos-SLQ}
\label{app:lanczos}
\subsection{Motivation}

Directly forming the operator
\[
M(\theta) = \mathbb{E}_x \big[J(x;\theta)J(x;\theta)^\top\big]
\]
is infeasible for modern neural networks due to its dimensionality.
However, our analysis only requires access to spectral information and
low-rank projections of $M$ and its directional derivative.

The block Lanczos method provides an efficient way to approximate the action
of $M$ on a low-dimensional Krylov subspace using only matrix--vector products.
This makes it particularly well-suited for probing spectral structure in
high-dimensional learning systems.

\subsection{Matrix--vector products via automatic differentiation}

For any vector $v$ in error space, the product $Mv$ can be written as
\[
Mv = \mathbb{E}_x \big[J(x) (J(x)^\top v)\big].
\]
This can be evaluated without explicitly forming $J$ by composing a
vector--Jacobian product (VJP) with a Jacobian--vector product (JVP).

In practice, for each mini-batch we compute
\[
u = J(x)^\top v, \qquad w = J(x)u,
\]
using automatic differentiation.
A small number of fixed mini-batches is sufficient to obtain a stable Monte
Carlo estimate of $Mv$ for the purposes of spectral diagnostics.

\subsection{Block Lanczos projection}

\begin{algorithm}[H]
\caption{Block Lanczos spectral projection}
\label{alg:block_lanczos}
\begin{algorithmic}[1]
\Require Matrix--vector product oracle $v \mapsto Mv$, block size $b$, steps $k$
\State Initialize block $V_0 \in \mathbb{R}^{d \times b}$ and orthonormalize
\For{$j=1$ to $k$}
  \State $W \gets M V_{j-1}$
  \State Orthogonalize $W$ against $\{V_0,\dots,V_{j-1}\}$
  \State QR factorization $W = V_j R_j$
\EndFor
\State Assemble block tridiagonal matrix $T$
\State Compute eigenpairs $(\lambda_i, u_i)$ of $T$
\end{algorithmic}
\end{algorithm}

\subsection{Ritz basis and projected dynamics}

The eigenvectors of the block tridiagonal matrix $T$ define a Ritz basis that
approximates the true eigenmodes of $M$ within the Krylov subspace.
Expressing operators in this basis yields a meaningful discretization of
spectral dynamics while remaining computationally tractable.

All subsequent diagnostics, including shell binning and Toeplitz analysis,
are performed in this approximate spectral basis.

\subsection{Projected operator derivative}

To probe spectral transport, we compute a finite-difference approximation of
the directional derivative of $M$ along the training trajectory:
\[
\dot M \approx \frac{M(\theta - \varepsilon \nabla \mathcal L) - M(\theta)}{\varepsilon}.
\]
Projecting $\dot M$ into the Ritz basis yields the operator
\[
\Omega = U^\top (Q^\top \dot M Q) U,
\]
which captures how spectral modes exchange energy under training.
This operator forms the basis of all shell-level coupling measurements.

\subsection{Shell binning and Toeplitz diagnostics}

Eigenvalues are grouped into logarithmic spectral shells.
The projected operator $\Omega$ is coarse-grained into a shell coupling matrix
$K$ by averaging squared couplings within each shell pair.

Toeplitz structure is assessed by comparing $K$ to its best Toeplitz
approximation, obtained by averaging along shell offsets.
This diagnostic directly tests the approximate shift-invariance across shells
assumed in Condition~\ref{cond:log-shift-invariance}.

\end{document}